\documentclass[letterpaper]{article} 
\usepackage[preprint]{aaai2027}  
\usepackage[hyphens]{url}  
\usepackage{graphicx} 
\usepackage{natbib}  
\usepackage{caption} 
\usepackage{algorithm}
\usepackage{algorithmic}

\usepackage{newfloat}
\usepackage{listings}
\DeclareCaptionStyle{ruled}{labelfont=normalfont,labelsep=colon,strut=off} 
\floatstyle{ruled}
\newfloat{listing}{tb}{lst}{}
\floatname{listing}{Listing}

\usepackage{booktabs}
\usepackage{amsfonts}       
\usepackage{nicefrac}       
\usepackage{microtype}      
\usepackage{xcolor}         
\usepackage{amsthm,amsmath,amssymb,amsfonts}
\usepackage{manyfoot}
\usepackage{xcolor}
\usepackage{subcaption}
\usepackage{booktabs}
\usepackage{lipsum}
\usepackage{multirow}
\usepackage{bm}
\usepackage{comment}
\usepackage{lscape}

\usepackage{physics}
\theoremstyle{thmstyleone}
\newtheorem{theorem}{Theorem}

\newtheorem{lemma}{Lemma}

\theoremstyle{thmstyletwo}

\theoremstyle{thmstylethree}

\newtheorem{assumption}{Assumption}

\newcommand{\NAMEA}{\textsc{D-FROST}}

\definecolor{cadmiumgreen}{rgb}{0.0, 0.42, 0.24}
\definecolor{frenchlilac}{rgb}{0.53, 0.38, 0.56}
\definecolor{navyblue}{rgb}{0.0, 0.0, 0.5}
\definecolor{amaranth}{rgb}{0.9, 0.17, 0.31}
\definecolor{brightmaroon}{rgb}{0.76, 0.13, 0.28}
\definecolor{azure(colorwheel)}{rgb}{0.0, 0.5, 1.0}
\definecolor{asparagus}{rgb}{0.53, 0.66, 0.42}
\definecolor{amethyst}{rgb}{0.6, 0.4, 0.8}

\title{D-FROST: Decentralized Federated pRompt-tuning via Optimal tranSporT for Non-IID and Imbalanced Data}
\author{
   Quan Minh Nguyen\textsuperscript{\rm 1}\equalcontrib,
    Hoang M. Ngo\textsuperscript{\rm 1}\equalcontrib,
    Trong Nghia Hoang\textsuperscript{\rm 2},
    My T. Thai\textsuperscript{\rm 1}\corresponding
}
\affiliations{
    \textsuperscript{\rm 1}University of Florida, FL, USA\\
    \textsuperscript{\rm 2}Washington State University, WA, USA\\

    \corresponding Correspondence to: mythai@cise.ufl.edu
}

\begin{document}

\maketitle







\begin{abstract}
Prompt tuning provides a parameter-efficient way to adapt foundation models (FMs) by freezing the pretrained backbone and updating only a small set of learnable prompts. This property makes prompt tuning especially suitable for decentralized federated learning (DFL), where exchanging full-model updates can be prohibitively expensive. However, prompt tuning in DFL introduces new challenges. Prompt sets learned from heterogeneous local data may not be index-wise aligned, making standard decentralized averaging unsuitable. In addition, the algorithm should be theoretically guaranteed to achieve consensus and make progress toward the shared objective. In this work, we provide the first study of prompt tuning in DFL. We formulate decentralized prompt tuning as a Wasserstein-based optimization problem over prompt measures, which captures the set-valued structure of prompts. We then propose \NAMEA, an optimal-transport-based (OT-based) decentralized prompt-tuning algorithm that merges neighborhood prompts into compact representative prompt sets through transportation-based matching. We further analyze \NAMEA\ by bounding the Wasserstein consensus error across clients, and establishing convergence of the network-level prompt barycenter to a neighborhood of stationarity. Experiments under heterogeneous client data demonstrate the effectiveness of \NAMEA\ for decentralized prompt tuning.
\end{abstract}

\section{Introductions}
\label{sec:intro}
Foundation models (FMs) have become a dominant foundation for modern AI systems, but adapting them to downstream tasks remains costly when all model parameters must be fine-tuned. Prompt tuning provides a parameter-efficient alternative by freezing the pretrained backbone and optimizing only a small set of learnable prompts~\citep{li2021prefix,lester2021power}. This makes prompt tuning particularly attractive in federated learning (FL), where data are distributed across clients and communication costs are a major bottleneck. Instead of transmitting full model updates, federated prompt tuning only exchanges lightweight prompt parameters. As state-of-the-art models now increasingly rely on  on fine-tuning foundation models like LLMs and Vision Transformers, prompt tuning in FL has emerged as a promising way to fine-tune large pretrained models over distributed data without centralizing raw information.~\citep{zhao2022fedprompt,che2023federated,weng2024probabilistic}. 

Decentralized federated learning (DFL) is a server-free variant of FL where clients communicate only with graph neighbors. For pre-trained model adaptation in DFL, prompt tuning offers a natural way to reduce communication by exchanging only lightweight prompt parameters. However, prompt tuning in DFL introduces two challenges. First, a prompt-tuning algorithm in DFL must ensure convergence, where local prompt states reach consensus and the network-level model progresses toward the shared objective. To the best of our knowledge, no prior work studies convergence of prompt tuning in DFL. Existing convergence studies for full-model DFL are not directly applicable because they rely on coordinate-aligned parameter vectors and Euclidean averaging, while prompt-tuning DFL operates on unordered prompt sets~\citep{yuan2016convergence,lian2017can,tang2018d2,koloskova2020unified}. Second, prompts learned from heterogeneous data may not be index-wise aligned, leading to a prompt misalignment issue where directly averaging prompts by index can merge unrelated prompt directions. In literature, PFPT~\citep{weng2024probabilistic} addresses this issue in centralized FL through probabilistic prompt aggregation. However, extending this idea to decentralized communication is nontrivial, since each client only observes local neighborhood prompts rather than a global collection.

\noindent \textbf{Contributions.} To the best of our knowledge, this is the first work to study prompt tuning in decentralized federated learning. The key contributions and insights of this work are summarized as follows:
\begin{itemize}
    \item[(i)] We formulate \emph{decentralized prompt tuning} as a Wasserstein-based optimization problem over prompt measures. This formulation preserves the standard DFL goal of learning a shared model state, while replacing Euclidean parameter consensus with Wasserstein prompt-measure consensus. As a result, it naturally captures the set-valued structure of client prompts.

    \item[(ii)] We propose \NAMEA, an OT-based decentralized prompt-tuning algorithm. Each client first updates its local prompts and then applies an OT-based \texttt{Merge} function to summarize neighborhood prompts into a compact representative prompt set. This merge operator avoids direct index-wise averaging and addresses prompt misalignment by matching prompts according to their geometry in the prompt embedding space.

    \item[(iii)] We provide a convergence analysis of \NAMEA. We first show that the local OT-based \texttt{Merge} solver becomes stable as the number of inner OT steps increases. We then establish that \NAMEA\ controls the Wasserstein consensus error across clients and that the network-level prompt barycenter converges to a neighborhood of stationarity for the shared prompt-tuning objective.

    \item[(iv)] We empirically evaluate \NAMEA\ against various decentralized federated prompt-tuning baselines based on existing DFL techniques. Through extensive experiments across a combination of eight diverse vision datasets, our results consistently show that our method is effective in data imbalance and extremely heterogeneous scenarios in  decentralized federated prompt-tuning.
\end{itemize}

\section{Related Works}
\label{sec:related}

\subsection{Prompt Tuning and Federated Prompt Tuning}

Prompt tuning aims to adapt pretrained models by optimizing a small set of learnable prompt parameters while keeping the backbone model frozen. Early representative works include prefix tuning, which optimizes continuous prefixes for generation tasks~\citep{li2021prefix}, and soft prompt tuning, which learns task-specific continuous prompts and becomes competitive with full fine-tuning as model scale increases~\citep{lester2021power}. 

Recent works extend prompt tuning to federated learning. FedPrompt aggregates prompt parameters rather than full models to reduce communication and storage costs~\citep{zhao2022fedprompt}, while PFPT uses probabilistic prompt aggregation to address non-IID and imbalanced data~\citep{weng2024probabilistic}.
However, these methods rely on centralized server aggregation and do not consider decentralized communication among graph neighbors. Moreover, heterogeneous clients may learn unaligned prompt sets, making index-wise averaging prone to combining mismatched prompt directions.

\subsection{Decentralized Federated Learning}

Unlike centralized FL, decetralized FL removes the server and lets clients communicate only with their neighbors over a graph. 
Early methods, such as distributed subgradient and decentralized gradient descent, combine local optimization with neighbor averaging~\citep{yuan2016convergence}. 
Later decentralized SGD analyses established competitive convergence under suitable mixing conditions~\citep{lian2017can,tang2018d2,koloskova2020unified}. DFedAvgM~\cite{sun2022decentralized} adapted the FedAvg approach of multiple local SGD iterations to the decentralized setting. DFedSAM~\citep{shi2023improving} employed the
sharpness-aware minimization optimizer to reduce the in
consistency of local models. NTK-DFL~\citep{thompson2025ntk} improves robustness to data heterogeneity through neural tangent kernel dynamics, but does not scale well to CNNs or Transformers.

Most decentralized learning methods assume that client states are coordinate-aligned model parameter vectors, so \texttt{Merge} is  implemented by weighted averaging through a mixing matrix. This assumption does not hold for decentralized prompt tuning, where prompts form unordered, potentially misaligned sets.
Thus, our work replaces parameter averaging with OT-based merging for Wasserstein consensus.


\section{Preliminaries}
\label{sec:preliminaries}
\subsection{Decentralized Federated Learning (DFL)}
\label{subsec:decentralized_fl}

DFL considers a network of clients that collaboratively optimize a learning objective without relying on a central server. The clients are connected through a communication graph $G=(V,E)$, where each node $u\in V$ represents a client and each edge $(u,v)\in E$ indicates direct communication. Each client $u$ owns a private dataset $D_u$, and the data distributions can be heterogeneous across clients.

Let $\mathcal{N}(u)=\{v\in V:(u,v)\in E\}$ denote the neighbor set of client $u$. The communication topology is often represented by a mixing matrix $W\in\mathbb{R}^{m\times m}$, where $W_{uv}>0$ only if $v=u$ or $v\in\mathcal{N}(u)$. The graph connectivity is characterized by
$
    \rho
    :=
    \left\|
    W-\frac{1}{m}\mathbf{1}\mathbf{1}^{\top}
    \right\|_2,
$
where $\rho<1$ for a connected graph. A smaller $\rho$ indicates faster information mixing and stronger consensus among clients.

A decentralized learning round consists of two steps: \texttt{LocalUpdate} and \texttt{Merge}. Each client first updates its local state using private data, then exchanges states with neighboring clients and aggregates the received information. 
In classical full-model decentralized training, the local state is the parameter vector of a shared architecture. Accordingly, \texttt{LocalUpdate} typically performs one or more stochastic gradient steps, while \texttt{Merge} applies mixing-matrix-weighted averaging over neighboring parameters~\citep{lian2017can,tang2018d2,koloskova2020unified}.

In this work, we study DFL with prompt tuning, where the pretrained backbone is frozen and only a small set of learnable prompt parameters is updated. Therefore, the client state is a prompt set rather than the full model parameter vector, and both \texttt{LocalUpdate} and \texttt{Merge} take different forms from full-model decentralized training.

\subsection{Measure Space and Wasserstein Distance}
\label{subsec:measure_wasserstein}

In our setting, each client maintains a set of learnable prompts as its state. A more natural view is to treat each prompt set as a distribution over the prompt embedding space. Under this view, comparing two prompt sets becomes a problem of comparing two probability measures.

Let $\mathcal{P}(\mathbb{R}^d)$ denote the space of probability measures on $\mathbb{R}^d$, and let $\mathcal{P}_2(\mathbb{R}^d)$ denote the subset of probability measures with finite second moment:
\begin{align*}
    \mathcal{P}_2(\mathbb{R}^d)
    :=
    \left\{
    \mu\in\mathcal{P}(\mathbb{R}^d)
    :
    \int_{\mathbb{R}^d}
    \|x\|^2
    d\mu(x)
    <
    \infty
    \right\}.
\end{align*}

This space provides a geometric setting for studying distributions supported in a Euclidean embedding space, such as prompt embeddings.
Given two probability measures $\mu,\nu\in\mathcal{P}_2(\mathbb{R}^d)$, a coupling between them is a joint probability measure $\pi\in\mathcal{P}(\mathbb{R}^d\times\mathbb{R}^d)$ whose marginals are $\mu$ and $\nu$. We denote the set of all such couplings by
\begin{align*}
    \Pi(\mu,\nu)
    :=
    \{
    \pi\in\mathcal{P}(\mathbb{R}^d\times\mathbb{R}^d)
    :
&\pi(A\times\mathbb{R}^d)=\mu(A),
    \; \\
    &\pi(\mathbb{R}^d\times B)=\nu(B) \}.
\end{align*}
The squared 2-Wasserstein distance between $\mu$ and $\nu$ is defined as
\begin{align*}
    W_2^2(\mu,\nu)
    :=
    \inf_{\pi\in\Pi(\mu,\nu)}
    \int_{\mathbb{R}^d\times\mathbb{R}^d}
    \|x-y\|^2
    d\pi(x,y).
\end{align*}
Intuitively, $W_2^2(\mu,\nu)$ measures the minimum transportation cost required to move the mass of $\mu$ to match $\nu$ under the squared Euclidean cost.
For empirical measures,
\begin{align*}
    \mu=\sum_{a=1}^{N} r_a\delta_{x_a},
    \qquad
    \nu=\sum_{i=1}^{M} c_i\delta_{y_i},
\end{align*}
where $r_a,c_i\ge 0$, $\sum_{a=1}^{N}r_a=1$, and $\sum_{i=1}^{M}c_i=1$, the coupling can be represented by a transport matrix $P\in\mathbb{R}_{+}^{N\times M}$. The feasible set is
\begin{align*}
    \Pi(r,c)
    :=
    \left\{
    P\in\mathbb{R}_{+}^{N\times M}
    :
    P\mathbf{1}_M=r,
    \;
    P^\top\mathbf{1}_N=c
    \right\}.
\end{align*}
In this discrete case, the squared 2-Wasserstein distance is:
\begin{align*}
    W_2^2(\mu,\nu)
    =
    \min_{P\in\Pi(r,c)}
    \sum_{a=1}^{N}
    \sum_{i=1}^{M}
    P_{ai}
    \|x_a-y_i\|^2.
\end{align*}

Therefore, Wasserstein distance compares two empirical distributions by optimizing over all possible matchings between their support points, rather than assuming a fixed ordering. This property is particularly useful for prompt sets, where the elements are not naturally ordered and can be misaligned across clients.

\section{Decentralized Wasserstein Prompt Tuning}
\label{sec:ot_decentralized}
In this section, we first introduce the decentralized prompt tuning setup. Then, we define the global objective as learning a shared prompt measure in the Wasserstein space. 
Finally, we propose an OT-based algorithm to approximately solve the decentralized prompt tuning problem.

\subsection{Setup}

Let $G=(V,E)$ be an undirected communication graph with $|V|=M$ clients. Each client $u\in V$ owns a private local dataset $D_u$. The clients collaboratively adapt a pretrained backbone model $F$ while keeping the backbone parameters fixed. Therefore, each client only maintains and updates a local prompt set.

At communication round $t$, client $u$ maintains
\begin{align*}
    \omega_u^{(t)}
    =
    \big\{
    \omega_{u1}^{(t)},\dots,\omega_{un}^{(t)}
    \big\},
    \qquad
    \omega_{ui}^{(t)}\in\mathbb{R}^d,
\end{align*}
where $n$ is the number of prompts and $d$ is the prompt dimension. We view this set as an empirical probability measure
\begin{align*}
    \mu_u^{(t)}
    :=
    \frac{1}{n}
    \sum_{i=1}^{n}
    \delta_{\omega_{ui}^{(t)}}.
\end{align*}
This representation treats the prompt state as an unordered set of support points in the prompt embedding space.

Let
$\mathcal{N}(u)
    =
    \{v\in V:(u,v)\in E\}
$
denote the neighbor set of client $u$. Since each client updates prompts using its own data distribution, neighboring prompt sets may become misaligned. Wasserstein distance provides a natural way to compare such prompt measures because it compares sets through optimal transport rather than assuming index-wise correspondence.

\subsection{Decentralized Wasserstein Prompt Tuning Objective}
\label{subsec:dwpt_objective}

As in classical DFL, our goal is to learn one shared model state. In our setting, this shared state is a prompt measure rather than a full model parameter vector. Let $\mathcal{P}_2(\mathbb{R}^d)$ denote the space of probability measures with finite second moment over the prompt embedding space. We define the global prompt-tuning objective as
\begin{equation}
    \min_{\mu\in\mathcal{P}_2(\mathbb{R}^d)}
    \quad
    \mathcal{F}(\mu)
    :=
    \frac{1}{M}
    \sum_{u=1}^{M}
    \mathcal{L}_u(\mu),
    \label{eq:global_prompt_measure_objective}
\end{equation}
where $\mathcal{L}_u(\mu)$ is the local prompt-tuning loss of client $u$ evaluated at prompt measure $\mu$.

In practice, there is no central server that directly maintains the shared prompt measure $\mu$. Instead, each client $u$ maintains a local empirical prompt measure
$
    \mu_u^{(t)}
    =
    \frac{1}{n}
    \sum_{i=1}^{n}
    \delta_{\omega_{ui}^{(t)}}.
$
These local prompt measures can be viewed as decentralized approximations of the shared prompt measure in \eqref{eq:global_prompt_measure_objective}.

To describe the collective state of the network, we use a network-level prompt barycenter, denoted by $\mu_{\mathrm{avg}}^{(t)}$. Formally, it can be viewed as a Wasserstein barycenter of the local prompt measures:
\begin{align}
    \mu_{\mathrm{avg}}^{(t)}
    \in
    \arg\min_{\mu\in\mathcal{P}_2(\mathbb{R}^d)}
    \frac{1}{M}
    \sum_{u=1}^{M}
    W_2^2
    \big(
    \mu,\mu_u^{(t)}
    \big).
    \label{eq:wasserstein_network_barycenter}
\end{align}
This barycenter provides a useful analytical object for describing the collective behavior of the decentralized system.

Under this view, the objective of decentralized prompt tuning follows the same two-fold principle as classical DFL. First, the decentralized trajectory should make progress toward minimizing the shared global objective in \eqref{eq:global_prompt_measure_objective}. 
Second, the local client states should achieve network consensus. In our setting, consensus means that the local prompt measures become close in Wasserstein distance.
We therefore measure the disagreement among local prompt measures by the Wasserstein consensus error
\begin{equation}
    \varepsilon^{(t)}
    :=
    \frac{1}{M}
    \sum_{u=1}^{M}
    W_2^2
    \big(
    \mu_u^{(t)},
    \mu_{\mathrm{avg}}^{(t)}
    \big).
    \label{eq:consensus_error_def}
\end{equation}
A smaller value of $\varepsilon^{(t)}$ indicates that the decentralized local prompt measures are more tightly concentrated around the network-level prompt barycenter, and hence that the clients have better prompt-level consensus.


This problem definition is natural for decentralized prompt tuning for two reasons. First, it preserves the standard DFL objective structure. In details, the target is still a single shared model state that minimizes the average client loss, rather than a separate personalized objective for each client. Second, it replaces Euclidean parameter consensus with Wasserstein prompt-measure consensus, which is more appropriate for set-valued prompt states. Since prompt indices across clients may not be aligned, Wasserstein distance compares prompt sets through optimal matching instead of forcing coordinate-wise or index-wise correspondence.

\subsection{OT-Based Decentralized Algorithm}
\label{subsec:ot_based_decentralized_algorithm}

We now propose an OT-based decentralized algorithm, named \NAMEA, for solving the decentralized prompt-tuning problem. The algorithm aims to make progress on the shared objective in \eqref{eq:global_prompt_measure_objective} while maintaining Wasserstein consensus among the local prompt measures, as measured by \eqref{eq:consensus_error_def}. The details are summarized in Algorithm~\ref{alg:dec_ot}. In general, each communication round of \NAMEA\ consists of two steps: a local prompt update and a neighbor merging step. The local update makes progress on the local loss, while the neighbor merging step promotes consensus among local prompt measures.

We first describe the local prompt update. 
\begin{equation}
    \tilde{\omega}_u^{(t)}
    \leftarrow
    \texttt{LocalUpdate}
    \big(
    F,D_u,\omega_u^{(t-1)}
    \big).
    \label{eq:local_prompt_update}
\end{equation}
This step produces the locally adapted prompt set $\tilde{\omega}_u^{(t)}$ before neighbor communication.

Next, client $u$ exchanges $\tilde{\omega}_u^{(t)}$ with its neighbors and forms the neighborhood prompt collection
\begin{equation}
    \Omega_u^{(t)}
    :=
    \tilde{\omega}_u^{(t)}
    \biguplus
    \biguplus_{v\in\mathcal{N}(u)}
    \tilde{\omega}_v^{(t)},
    \qquad
    N_u^{(t)}
    :=
    \big|\Omega_u^{(t)}\big|.
    \label{eq:Omega_u}
\end{equation}
The collection $\Omega_u^{(t)}$ contains prompt information from client $u$ and its neighbors. This exchange supports decentralized consensus by incorporating neighborhood knowledge. However, retaining all received prompts can increase storage and computation costs, so $\Omega_u^{(t)}$ should be summarized by a compact prompt set that preserves the essential information.

In \NAMEA, this compact set is obtained through an OT-based prompt merging problem. The merge step approximates the collection $\Omega_u^{(t)}$ by a compact empirical measure under a transportation-based geometry. Thus, the OT-based \texttt{Merge} function summarizes received prompts while supporting Wasserstein consensus among local prompt measures.

We now describe the OT-based prompt merging problem. Given $
    \Omega_u^{(t)}
    =
    \big\{
    z_a
    \big\}_{a=1}^{N_u^{(t)}}$,
the goal is to construct a representative prompt set
$$
    \Phi
    =
    \{\phi_i\}_{i=1}^{n}, \qquad
    \phi_i\in\mathbb{R}^d.
$$

We introduce a transport plan
$P\in\mathbb{R}_{+}^{N_u^{(t)}\times n}$, where $P_{ai}$ measures how much the neighborhood prompt $z_a$ contributes to the representative prompt $\phi_i$. The matching cost is
\begin{align}
    C_{ai}(\Phi)
    :=
    \frac{1}{2\sigma^2}
    \left\|
    z_a-\phi_i
    \right\|^2,
    \label{eq:ot_cost}
\end{align}
where $\sigma^2$ controls the spatial scale of the cost.

We assign uniform source mass to the neighborhood prompts:
\begin{align*}
    q
    :=
    \frac{1}{N_u^{(t)}}
    \mathbf{1}_{N_u^{(t)}},
\end{align*}
and impose the source marginal constraint
$
    P\mathbf{1}_n=q
$.
This ensures that every neighborhood prompt participates in the merge. We do not impose a fixed target marginal over the representative prompts, so different merged prompts can receive different amounts of mass based on the geometry of $\Omega_u^{(t)}$.

Client $u$ computes the merged prompt set by solving
{\small
\begin{align}
    \min_{P\ge 0,\ \Phi}
    \quad
    \mathcal{J}_u(P,\Phi)
    :=
    &\underbrace{
    \langle P,C(\Phi)\rangle
    }_{\text{spatial matching}}
    +
    \varepsilon
    \underbrace{
    \sum_{a=1}^{N_u^{(t)}}
    \sum_{i=1}^{n}
    P_{ai}(\log P_{ai}-1)
    }_{\text{entropy regularization}} \nonumber
    \\ &+
    \underbrace{
    \frac{\lambda}{2\sigma^2}
    \sum_{i=1}^{n}
    \|\phi_i\|^2
    }_{\text{$L_2$ regularization}}
    \quad
    \text{s.t.}
    \quad
    P\mathbf{1}_n=q.
    \label{eq:full_objective}
\end{align}
}

Here, $\varepsilon>0$ controls the softness of the assignments, and $\lambda>0$ controls the regularization strength. The spatial term matches neighborhood prompts to representative prompts, the entropy term produces soft transport assignments, and the $L_2$ term stabilizes the representatives. 

The objective in \eqref{eq:full_objective} is block-wise tractable. Fixing $\Phi$ gives a closed-form update for $P$, and fixing $P$ gives a closed-form update for $\Phi$. Thus, each client solves the OT merge problem by alternating between transport and barycenter updates.

\paragraph{Transport step.}
Given the current representative prompts $\Phi^{(s-1)}$, client $u$ updates the transport plan. Since the source marginal constraint fixes only the row sums of $P$, the rows of $P$ decouple, yielding the closed-form update
\begin{equation}
    P_{ai}^{(s)}
    =
    \frac{1}{N_u^{(t)}}
    \frac{
    \exp\big(-C_{ai}(\Phi^{(s-1)})/\varepsilon\big)
    }{
    \sum_{j=1}^{n}
    \exp\big(-C_{aj}(\Phi^{(s-1)})/\varepsilon\big)
    }.
    \label{eq:P_subproblem}
\end{equation}
This update softly assigns each neighborhood prompt to the representative prompts. Smaller $\varepsilon$ gives sharper assignments, while larger $\varepsilon$ gives smoother mixing.

\paragraph{Barycenter step.}
Given the updated transport plan $P^{(s)}$, client $u$ updates each representative prompt by setting the gradient of \eqref{eq:full_objective} with respect to $\phi_i$ to zero, yielding
\begin{equation}
    \phi_i^{(s)}
    =
    \frac{
    \sum_{a=1}^{N_u^{(t)}}
    P_{ai}^{(s)}
    z_a
    }{
    \sum_{a=1}^{N_u^{(t)}}
    P_{ai}^{(s)}
    +
    \lambda
    }.
    \label{eq:update_phi}
\end{equation}
Thus, each representative prompt is a regularized weighted average of the neighborhood prompts assigned to it. This step moves the representatives toward dominant prompt directions without relying on index-wise averaging.

After $S$ alternating steps, client $u$ sets
\begin{align*}
    \omega_u^{(t)}
    \leftarrow
    \texttt{Merge}_{\mathrm{OT}}
    \big(
    \Omega_u^{(t)}
    \big)
    :=
    \Phi^{(S)}.
\end{align*}
This completes the OT-based \texttt{Merge} step. The overall procedure alternates between local adaptation, which improves the client-specific prompt loss, and OT-based neighbor merging, which promotes consensus among prompt measures. 

\begin{algorithm}[tb]
\caption{Decentralized Wasserstein Prompt Tuning}
\label{alg:dec_ot}
\begin{algorithmic}[1]
\REQUIRE Graph $G=(V,E)$, backbone $F$, local datasets $\{D_u\}_{u\in V}$, rounds $T$, OT steps $S$, number of prompts $n$, entropy weight $\varepsilon$, regularization weight $\lambda$, scale $\sigma^2$
\STATE Initialize local prompt sets $\omega_u^{(0)}$ for all $u\in V$
\FOR{$t=1$ \TO $T$}
  \FORALL{$u\in V$ \textbf{in parallel}}
    \STATE \textbf{Local update:}
    \[
        \tilde{\omega}_u^{(t)}
        \leftarrow
        \texttt{LocalUpdate}
        \big(
        F,D_u,\omega_u^{(t-1)}
        \big)
    \]
    \STATE \textbf{Neighbor exchange:} send $\tilde{\omega}_u^{(t)}$ to neighbors and receive $\{\tilde{\omega}_v^{(t)}:v\in\mathcal{N}(u)\}$
    \STATE \textbf{Neighborhood collection:}
    \[
        \Omega_u^{(t)}
        =
        \tilde{\omega}_u^{(t)}
        \biguplus
        \biguplus_{v\in\mathcal{N}(u)}
        \tilde{\omega}_v^{(t)}
    \]
    \STATE Write $\Omega_u^{(t)}=\{\omega_a\}_{a=1}^{N_u^{(t)}}$
    \STATE Initialize representative prompts $\Phi^{(0)}=\{\phi_i^{(0)}\}_{i=1}^{n}$ from $\Omega_u^{(t)}$ or from $\omega_u^{(t-1)}$
    \FOR{$s=1$ \TO $S$}
      \STATE Construct cost matrix $C(\Phi^{(s-1)})$ using \eqref{eq:ot_cost}
      \STATE Update transport plan $P^{(s)}$ using \eqref{eq:P_subproblem}
      \STATE Update representative prompts $\Phi^{(s)}$ using \eqref{eq:update_phi}
    \ENDFOR
    \STATE \textbf{Merge output:} set $\omega_u^{(t)}\leftarrow\Phi^{(S)}$
  \ENDFOR
\ENDFOR
\RETURN Final local prompt sets $\{\omega_u^{(T)}\}_{u\in V}$
\end{algorithmic}
\end{algorithm}

\section{Theoretical Analysis of \NAMEA}
\label{sec:theoretical_analysis}

In this section, we analyze the theoretical properties of \NAMEA. First, we show that the local OT-based \texttt{Merge} step becomes stable as the number of inner OT steps increases. Second, we analyze the global behavior of \NAMEA\ and show that the local prompt measures remain close to a network-level barycenter, which then converges to a neighborhood of Wasserstein stationarity for the shared prompt-tuning objective in \eqref{eq:global_prompt_measure_objective}. All proofs are deferred to appendix \ref{appendix:theoretical_analysis}.

\begin{figure*}[ht!]
  \centering
  \includegraphics[width=0.78\textwidth]{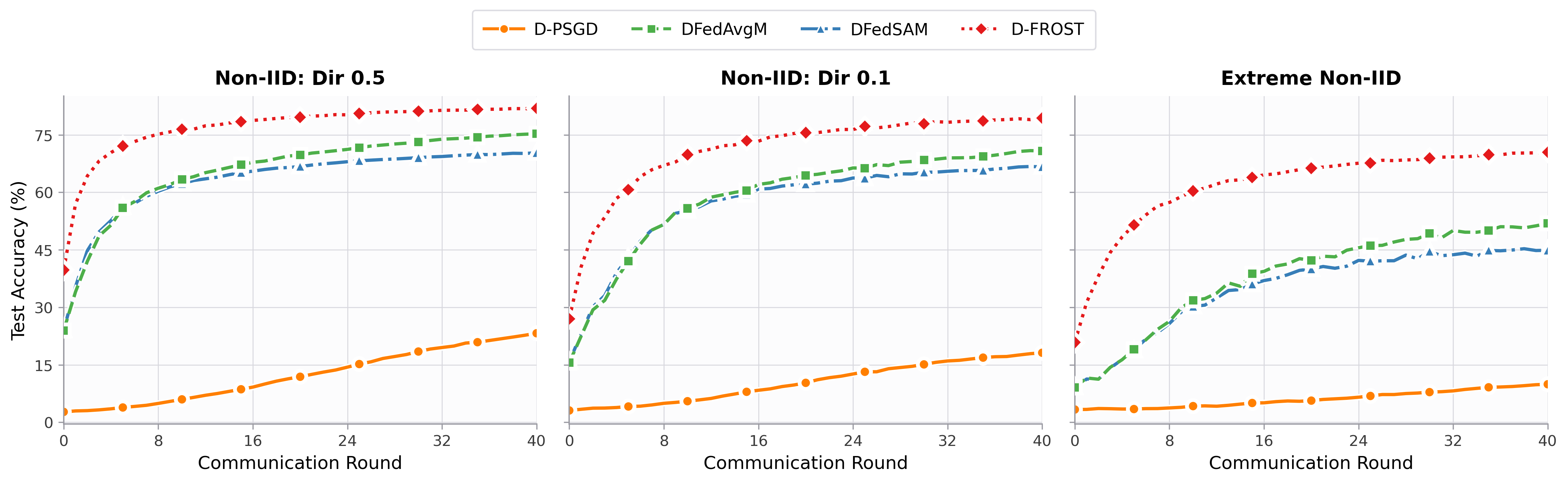}
  \caption{Test accuracy of all methods on \textbf{FiveDataset} (50 clients)
           under three non-IID settings.}
  \label{fig:noniid_fivedataset}
\end{figure*}

\subsection{Stability of the Local OT-Based Merge}
\label{subsec:local_convergence}

We first analyze the local OT-based \texttt{Merge} operator. Recall that after client $u$ forms the neighborhood prompt collection $\Omega_u^{(t)}$, it solves the local OT problem in \eqref{eq:full_objective} by alternating between the transport update and the barycenter update. Since the merged prompt set is used as the client state for the next communication round, the inner OT solver should produce a stable representative prompt set.

\begin{theorem}[Stability of the Alternating OT Solver]
\label{thm:convergence_rate}
Let $\Delta^{(s)}$ be defined as above, and let
$
    \mu
    :=
    \min
    \left(
    \varepsilon,
    \frac{\lambda}{\sigma^2}
    \right)
    >
    0.
$
After $S$ alternating OT steps, the minimum iterate difference is bounded by an $\mathcal{O}(1/\sqrt{S})$ rate:
\begin{align*}
    \min_{1\le s\le S}
    \Delta^{(s)}
    \le
    \sqrt{
    \frac{
    4
    \big(
    \mathcal{J}_u(P^{(0)},\Phi^{(0)})
    -
    \mathcal{J}^*
    \big)
    }{
    \mu S
    }
    }.
\end{align*}
\end{theorem}

Let $\Delta^{(s)}
    :=
    \|P^{(s)}-P^{(s-1)}\|_F
    +
    \|\Phi^{(s)}-\Phi^{(s-1)}\|_F$
denote the discrepancy between two consecutive inner iterations.
Theorem~\ref{thm:convergence_rate} provides a quantitative stability guarantee for the local OT-based \texttt{Merge} step. It shows that among the first $S$ alternating iterations, there exists an iterate whose change from the previous iterate $\Delta^{(s)}$ is bounded by $\mathcal{O}(1/\sqrt{S})$. Thus, increasing the number of inner OT steps makes the local merge output progressively more stable. The bound also shows that stability depends on the initial objective gap $\mathcal{J}_u(P^{(0)},\Phi^{(0)})-\mathcal{J}^*$ and the effective strong-convexity parameter $\mu=\min(\varepsilon,\lambda/\sigma^2)$. This local control is crucial because the approximation error of the OT-based merge directly affects network-level consensus and global convergence.

\subsection{Global Wasserstein Consensus and Stationarity}
\label{subsec:global_convergence}

We now analyze the global behavior of \NAMEA. The analysis establishes two main results. First, \NAMEA\ controls the consensus error $\varepsilon^{(t)}$ defined in \eqref{eq:consensus_error_def}, meaning that the local prompt measures remain close to a network-level barycenter $\mu_{\mathrm{avg}}^{(t)}$. Second, this barycenter converges to a neighborhood of stationarity for the shared prompt-tuning objective in \eqref{eq:global_prompt_measure_objective}.

For each client $u$, we represent its prompt set at round $t$ as the empirical measure $\mu_u^{(t)}$. We have $\mu_{\mathrm{avg}}^{(t)}$ as the Wasserstein barycenter of the local prompt measures, as defined in \eqref{eq:wasserstein_network_barycenter}. With the learning rate $\eta$, we model the local update as:
\begin{align*}
    \mu_u^{+(t)}
    =
    \big(
    I-\eta\nabla_{W_2}\mathcal{F}_u
    \big)_{\#}
    \mu_u^{(t-1)}.
\end{align*}
The subsequent OT-based \texttt{Merge} step approximates the ideal neighborhood barycenter using a compact representative prompt measure. We capture the approximation error of this finite-prompt OT merge by the following bounded-error:
\begin{align*}
    \mathbb{E}
    \left[
    \left\|
    \tilde{v}^{(t)}-v^{(t)}
    \right\|_{L^2(\mu_{\mathrm{avg}}^{(t-1)})}^{2}
    \right]
    \le
    \delta_n^2(S),
\end{align*}
where $\tilde{v}^{(t)}$ denotes the ideal displacement induced by the uncompressed barycenter, $v^{(t)}$ denotes the actual displacement induced by the OT-merged barycenter, and $\delta_n^2(S)$ captures the approximation error caused by the finite prompt budget $n$ and the finite number of OT steps $S$.

We use the following standard assumptions.

\begin{assumption}[Wasserstein Smoothness]
\label{assum:w2_smooth_rev}
Each local loss functional $\mathcal{F}_u(\mu)$ is $L$-smooth over the 2-Wasserstein space. Consequently, the global functional
$\mathcal{F}(\mu)=\frac{1}{M}\sum_{u=1}^{M}\mathcal{F}_u(\mu)$
is also $L$-smooth and bounded below by $\mathcal{F}^*>-\infty$.
\end{assumption}

\begin{assumption}[Bounded Gradient Variance]
\label{assum:ot_variance_error}
The local Wasserstein gradients have uniformly bounded second moment:
\begin{equation}
    \mathbb{E}
    \left[
    \|\nabla_{W_2}\mathcal{F}_u(\mu)\|_{L^2}^{2}
    \right]
    \le
    G^2.
\end{equation}
\end{assumption}

\begin{assumption}[Bounded Prompt Support]
\label{assum:bounded_support}
There exists a constant $D>0$ such that for all clients, prompts, and communication rounds,
\begin{equation}
    \|\omega_{ui}^{(t)}\|\le D.
\end{equation}
\end{assumption}

\begin{assumption}[Graph Mixing]
\label{assum:graph_mixing}
The communication matrix $W\in\mathbb{R}^{M\times M}$ is symmetric and doubly stochastic. Its mixing factor satisfies
\begin{equation}
    \rho
    :=
    \left\|
    W-\frac{1}{M}\mathbf{1}\mathbf{1}^{\top}
    \right\|_2
    <1.
\end{equation}
\end{assumption}

The next theorem states that the local prompt measures remain close to the network-level barycenter.

\begin{theorem}[Wasserstein Consensus Bound]
\label{thm:global_consensus_bound}
Suppose Assumptions~\ref{assum:w2_smooth_rev}--\ref{assum:graph_mixing}. the expected network consensus error $\varepsilon^{(t)}$ converges asymptotically to a stationary bounded neighborhood. Specifically, for any $t \to \infty$:
\begin{equation}
    \varepsilon^{(t)} \le \frac{\beta}{1-\alpha} = \mathcal{O}\left( \frac{\delta_n^2(S)}{1 - \mathcal{D}\rho^2} + \frac{\eta^2 G^2}{1 - \mathcal{D}\rho^2} \right),
\end{equation}
where $\alpha = 6\mathcal{D}\rho^2 (1+\eta^2L^2) < 1$, and $\beta = 6\delta_n^2(S) + 6\mathcal{D}\rho^2\eta^2G^2$, and $\mathcal{D}>0$ is a metric translation constant.
\end{theorem}

Theorem~\ref{thm:global_consensus_bound} shows that \NAMEA\ controls prompt disagreement in Wasserstein space. The consensus neighborhood has two sources. The first term, $\delta_n^2(S)$, is the approximation error introduced by representing the exchanged neighborhood prompts with a compact OT-merged prompt set. The second term, $\eta^2G^2$, is caused by heterogeneous local updates. The denominator depends on the graph mixing factor $\rho$. Specifically, better-connected graphs have smaller $\rho$ and therefore tighter consensus neighborhoods.

Finally, we state the stationarity result.

\begin{theorem}[Convergence to a Wasserstein Stationarity Neighborhood]
\label{thm:w2_stationarity_neighborhood}
Suppose Assumptions~\ref{assum:w2_smooth_rev}--\ref{assum:graph_mixing}. Let $\eta\le 1/L$ and $\alpha<1$. Then after $T$ communication rounds, \NAMEA\ satisfies
{\small
\begin{align*}
    &
    \frac{1}{T}
    \sum_{t=1}^{T}
    \mathbb{E}
    \left[
    \left\|
    \nabla_{W_2}\mathcal{F}
    \left(
    \mu_{\mathrm{avg}}^{(t-1)}
    \right)
    \right\|_{L^2(\mu_{\mathrm{avg}}^{(t-1)})}^{2}
    \right]
    \le
    \frac{
    2\left(
    \mathcal{F}(\mu_{\mathrm{avg}}^{(0)})-\mathcal{F}^*
    \right)
    }{\eta T}
    \\
    &+
    \mathcal{O}
    \left(
    \frac{\eta^2L^2G^2}{1-\mathcal{D}\rho^2}
    \right)
    +
    \mathcal{O}
    \left(
    \frac{L^2\delta_n^2(S)}{1-\mathcal{D}\rho^2}
    +
    \frac{\delta_n^2(S)}{\eta^2}
    \right).
\end{align*}
}
\end{theorem}

Theorem~\ref{thm:w2_stationarity_neighborhood} shows that \NAMEA\ converges to a neighborhood of Wasserstein stationarity. The first term is the standard optimization term and vanishes as $T$ increases. The second term reflects the effect of local gradient variance and graph-induced consensus error. The third term captures the approximation error of the OT-based \texttt{Merge} step. Thus, more accurate local merging, achieved by increasing the prompt budget $n$ or using more OT steps $S$, leads to a tighter stationarity neighborhood.

\begin{figure}[t!]
    \centering
    \includegraphics[width=0.35\textwidth]{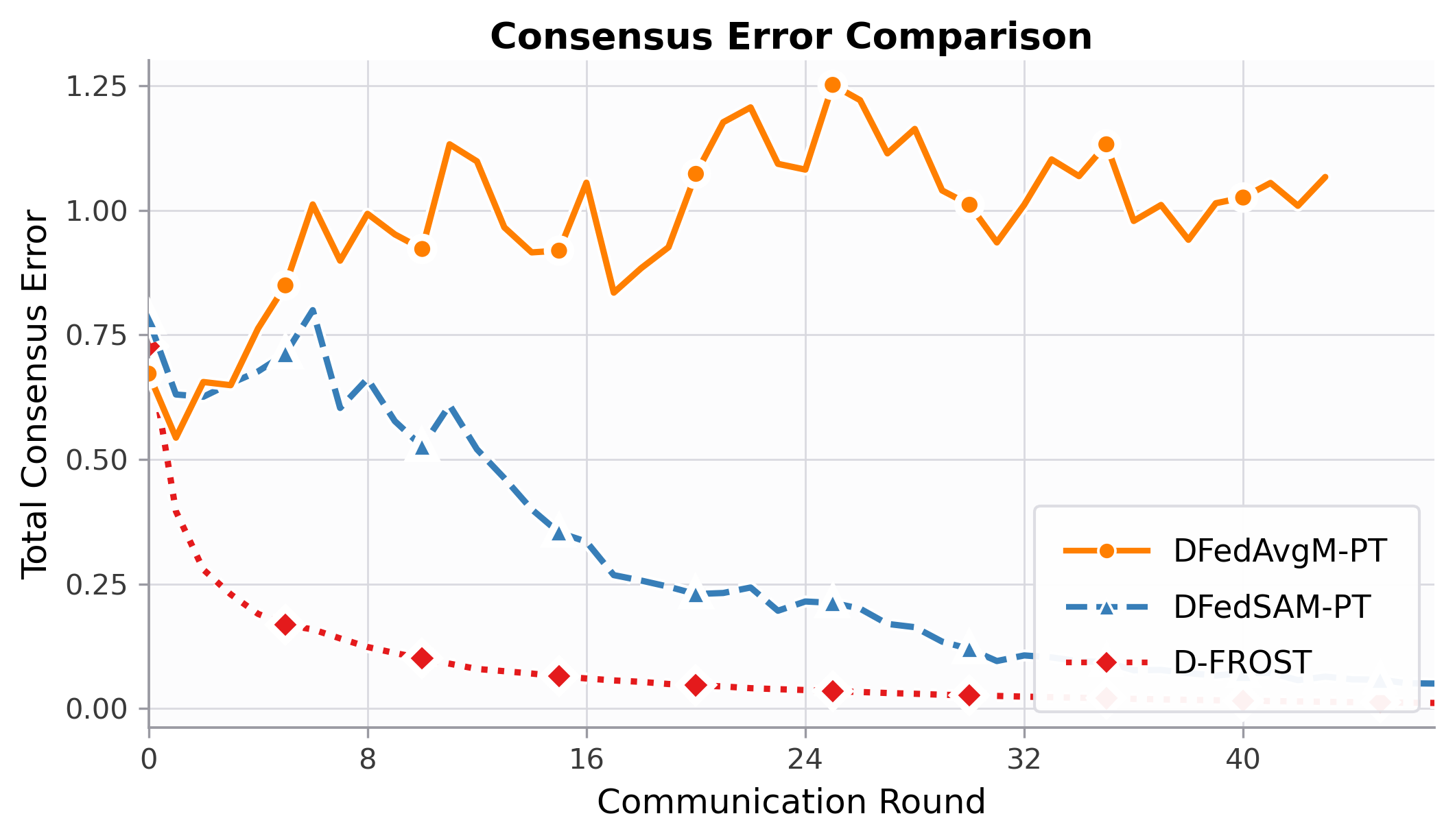}
    \caption{Total Wasserstein consensus error across all clients on FiveDataset under Dirichlet $\alpha = 0.1$.}
    \label{fig:total_concencus_err}
\end{figure}

\begin{figure}[t!]
    \centering
    \includegraphics[width=0.35\textwidth]{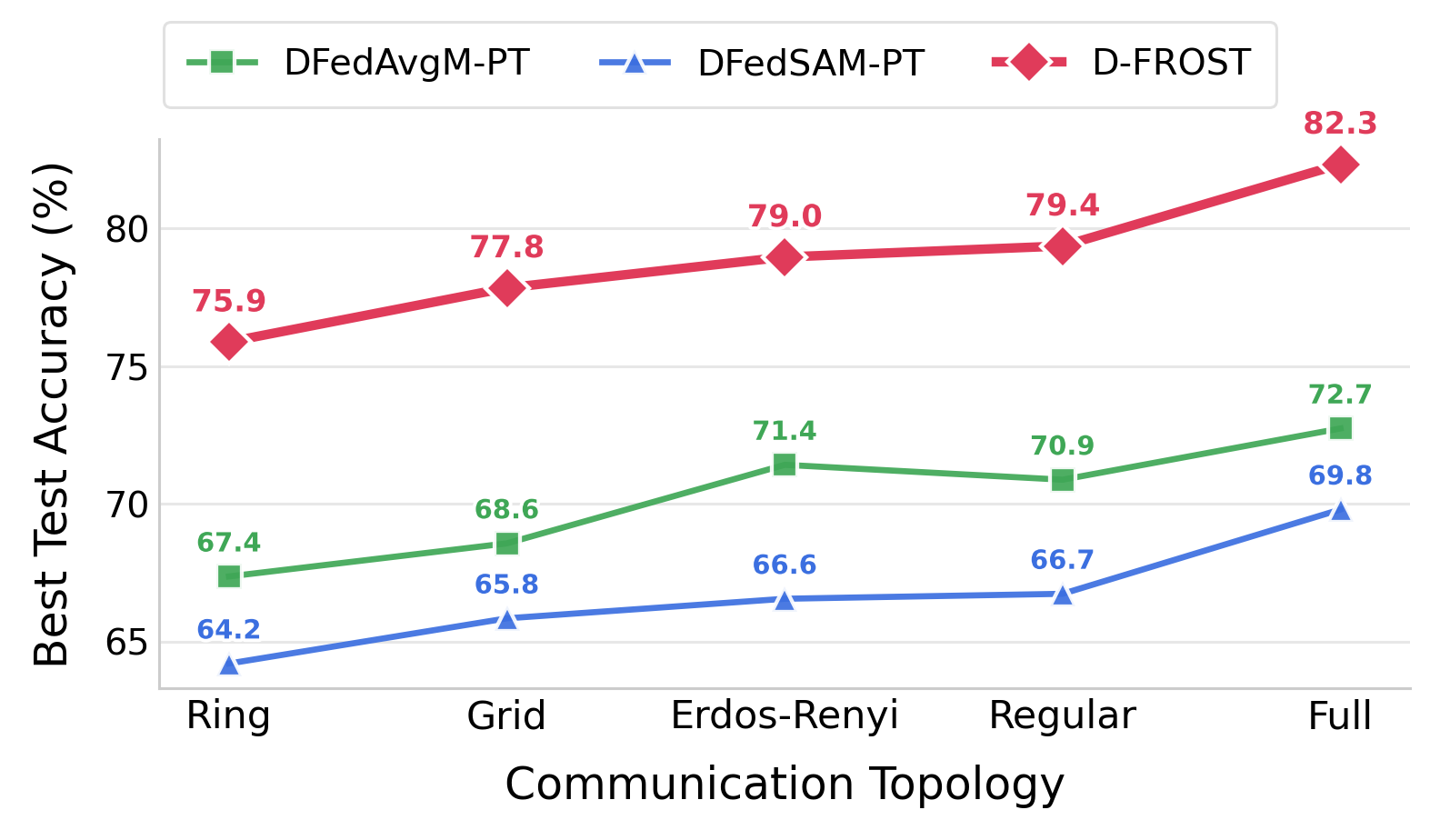}
    \caption{Topology-aware performance across communication topologies on FiveDataset under Dirichlet $\alpha = 0.1$.}
    \label{fig:topology}
\end{figure}

\begin{figure*}[t!]
    \centering
    \includegraphics[width=0.85\textwidth]{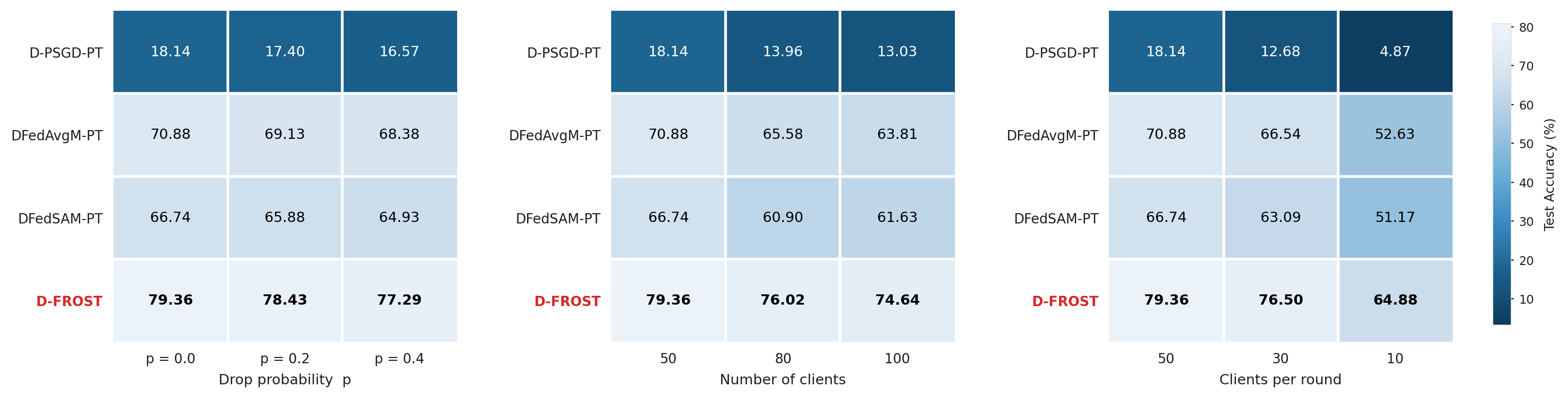}
    \caption{Robustness of D-FROST under varying network conditions on FiveDataset ($\alpha=0.1$). \textbf{Left:} increasing link drop probability $p$. \textbf{Center:} increasing the number of clients. \textbf{Right:} decreasing the number of clients sampled per round.}
    \label{fig:robustness}
\end{figure*}

\section{Experiments}
\label{sec:exp}
\subsection{Experiment Setup}\label{subsec:expsetup}
\subsubsection{Dataset and Data Partition} We induce data heterogeneity by pooling classification datasets from different visual domains. \textbf{FourDataset}~\citep{weng2024probabilistic}
combines MNIST-M, Fashion-MNIST, CINIC-10, and MMAFEDB and \textbf{FiveDataset}~\citep{wang2022learning}
combines CIFAR-10, MNIST, Fashion-MNIST, SVHN, and notMNIST. 
Their combination introduces substantial distributional heterogeneity.
We use two partition regimes. In the \textbf{Dirichlet split}, each client draws a
class-proportion vector $\mathbf{p}_u \sim \mathrm{Dirichlet}(\alpha \cdot \mathbf{1}_s)$
over its $s$ domain classes; smaller $\alpha$ means stronger skew, and we report
$\alpha=0.1$ and $\alpha=0.5$. In the \textbf{extreme non-IID split}, each client is
dominated by a single class ($99\%$ of its samples), with the remaining $1\%$ pooled and
spread across the others~\citep{weng2024probabilistic}. Both schemes are applied within each domain and the resulting client subsets merged. 

\subsubsection{Baselines}
We compare our method against three popular DFL baselines:  D-PSGD-PT~\citep{lian2017can}, DFedAvgM-PT~\citep{sun2022decentralized}, and DFedSAM-PT~\citep{shi2023improving}. For all methods, the pretrained backbone is frozen, while only the prompts and classification head are locally updated and communicated. 
\subsubsection{Implementation Details}
The number of clients is set to $40$ for \textbf{FourDataset} and $50$ for
\textbf{FiveDataset}, with all clients participating in every communication round.
We use the Adam optimizer with batch size $16$ and an initial learning rate of
$10^{-4}$. All methods run $5$ local epochs per communication round, except D-PSGD,
which runs only $1$. We use ViT-B/32 as the frozen
backbone, with 10 trainable prompt tokens of dimension $d = 768$ prepended
to the patch-embedding sequence before it enters the frozen transformer blocks.
\subsubsection{Communication Topologies} 
As our standard network topology, we employ a time-varying $\kappa$-regular graph where $\kappa = 5$ for \textbf{FiveDataset} and $\kappa = 4$ for \textbf{FourDataset}. Specifically, during each communication round $t$, we sample a new random graph $G^{(t)} = (V, E^{(t)})$ with a uniform degree of $\kappa$, meaning every client $u$ connects to exactly $|\mathcal{N}(u)| = \kappa$ neighbors. To  evaluate the algorithmic robustness across different network structures, we also benchmark performance on Ring, Grid, Erdős-Rényi, Regular Graph, and Fully connected topologies.
More details 
are given in Appx. \ref{appendix:additional_settings}.

\subsection{Experimental Results}\label{subsec:perfeval}

\paragraph{Performance and Convergence.}
We report the accuracy of D-FROST and the DFL baselines on FiveDataset under the two Dirichlet splits ($\alpha = 0.5$, $\alpha = 0.1$) and the extreme non-IID (imbalance) split in Figure~\ref{fig:noniid_fivedataset}. Across all settings, D-FROST consistently achieves the best accuracy and fastest convergence. 
Specifically, it reaches $81.95\%$, $79.36\%$, and $70.46\%$, yielding improvements of $6.67$, $8.48$, and $18.57$ points. The performance gap widens sharply as heterogeneity increases, particularly under the extreme non-IID split where index-wise prompt averaging often merges misaligned prompt directions. Detailed and additional results on FourDataset are given in Appx. \ref{appendix:additional_results}. 

Beyond final accuracy, Figure~\ref{fig:noniid_fivedataset} shows that D-FROST separates from the baselines within a few rounds and reaches target accuracy faster. Theorem~\ref{thm:w2_stationarity_neighborhood} supports this behavior by establishing an ($\mathcal{O}(1/T)$) convergence rate to a neighborhood determined by the merge error ($\delta_n^2(S)$) and gradient variance.
Meanwhile, index-wise averaging baselines ignores prompt misalignment, leading to larger consensus error under severe data skew and slower convergence (see Figure~\ref{fig:total_concencus_err}). We further study factors that impact the consensus error in Appx. \ref{appendix:consensus_factors}.

\paragraph{Topology-aware performance.}
Figure~\ref{fig:topology} evaluates accuracy across five topologies, ordered by decreasing sparsity (mixing factor $\rho$): Ring $>$ Grid $>$ Erdős-Rényi ($\mathbb{E}[\kappa] = 5$) $\approx$ Regular ($\kappa = 5$) $>$ Fully-connected. D-FROST consistently outperforms all baselines across every structure. As connectivity increases ($\rho$ decreases), performance improves for all methods. For D-FROST, this aligns with Theorem~\ref{thm:global_consensus_bound}: a smaller $\rho$ yields a tighter Wasserstein consensus neighborhood, enabling faster network agreement. Despite baseline improvements in denser networks, the substantial performance gap in favor of D-FROST persists throughout.

\paragraph{Robustness to network conditions.}
We further stress-test D-FROST along three axes that degrade decentralized training:
unreliable links, network scale, and partial participation. Figure~\ref{fig:robustness}
reports test accuracy as each factor is made harsher. Under link dropout (left), where
each edge fails with probability $p$, D-FROST declines only mildly from $79.36\%$ at $p=0.0$ to $77.29\%$ at $p=0.4$. As the number of clients in the network grows to 100 (center), D-FROST leads over the strongest
baseline by $11\%$. Under partial
participation (right): with only $10$ of the clients active per round, D-FROST retains
$64.88\%$, a margin of nearly $12$ points over DFedAvgM. 

\paragraph{Further Experiments.} We analyze the cost of D-FROST in Appx. ~\ref{appendix:comp_cost} and provide ablation studies that further analyze the impact of different deployment settings in Appx. ~\ref{appx:ablation}.

\section{Conclusion}\label{sec:conclusion} We presented the first study of prompt tuning in decentralized federated learning. We formulate it as a Wasserstein optimization over prompt measures and propose \textbf{D-FROST}, an optimal-transport-based algorithm that merges neighborhood prompts into a compact representative set without index-wise averaging. Theoretically, we proved that the local OT solver is stable, that the Wasserstein consensus error across clients stays within a bounded neighborhood, and that the network-level prompt barycenter converges to a neighborhood of stationarity for the shared objective. Across eight datasets, D-FROST consistently outperforms decentralized baselines, with the largest gains under extreme non-IID data.

\bibliography{aaai2027}

@inproceedings{lian2017can,
author = {Lian, Xiangru and Zhang, Ce and Zhang, Huan and Hsieh, Cho-Jui and Zhang, Wei and Liu, Ji},
title = {Can decentralized algorithms outperform centralized algorithms? a case study for decentralized parallel stochastic gradient descent},
year = {2017},
isbn = {9781510860964},
publisher = {Curran Associates Inc.},
address = {Red Hook, NY, USA},
booktitle = {Proceedings of the 31st International Conference on Neural Information Processing Systems},
pages = {5336–5346},
numpages = {11},
location = {Long Beach, California, USA},
series = {NIPS'17}
}

@InProceedings{tang2018d2,
  title = 	 {$D^2$: Decentralized Training over Decentralized Data},
  author =       {Tang, Hanlin and Lian, Xiangru and Yan, Ming and Zhang, Ce and Liu, Ji},
  booktitle = 	 {Proceedings of the 35th International Conference on Machine Learning},
  pages = 	 {4848--4856},
  year = 	 {2018},
  editor = 	 {Dy, Jennifer and Krause, Andreas},
  volume = 	 {80},
  series = 	 {Proceedings of Machine Learning Research},
  month = 	 {10--15 Jul},
  publisher =    {PMLR},
  url = 	 {https://proceedings.mlr.press/v80/tang18a.html}
}

@InProceedings{koloskova2020unified,
  title = 	 {A Unified Theory of Decentralized {SGD} with Changing Topology and Local Updates},
  author =       {Koloskova, Anastasia and Loizou, Nicolas and Boreiri, Sadra and Jaggi, Martin and Stich, Sebastian},
  booktitle = 	 {Proceedings of the 37th International Conference on Machine Learning},
  pages = 	 {5381--5393},
  year = 	 {2020},
  editor = 	 {III, Hal Daumé and Singh, Aarti},
  volume = 	 {119},
  series = 	 {Proceedings of Machine Learning Research},
  month = 	 {13--18 Jul},
  publisher =    {PMLR},
  url = 	 {https://proceedings.mlr.press/v119/koloskova20a.html}
}

@article{li2021prefix,
  title={Prefix-Tuning: Optimizing Continuous Prompts for Generation},
  author={Xiang Lisa Li and Percy Liang},
  journal={Proceedings of the 59th Annual Meeting of the Association for Computational Linguistics and the 11th International Joint Conference on Natural Language Processing (Volume 1: Long Papers)},
  year={2021},
  pages={4582-4597},
  url={https://api.semanticscholar.org/CorpusID:230433941}
}

@inproceedings{lester2021power,
  title={The Power of Scale for Parameter-Efficient Prompt Tuning},
  author={Brian Lester and Rami Al-Rfou and Noah Constant},
  booktitle={Conference on Empirical Methods in Natural Language Processing},
  year={2021},
  url={https://api.semanticscholar.org/CorpusID:233296808}
}

@INPROCEEDINGS{zhao2022fedprompt,
  author={Zhao, Haodong and Du, Wei and Li, Fangqi and Li, Peixuan and Liu, Gongshen},
  booktitle={ICASSP 2023 - 2023 IEEE International Conference on Acoustics, Speech and Signal Processing (ICASSP)}, 
  title={FedPrompt: Communication-Efficient and Privacy-Preserving Prompt Tuning in Federated Learning}, 
  year={2023},
  volume={},
  number={},
  pages={1-5},
  doi={10.1109/ICASSP49357.2023.10095356}}

@inproceedings{
che2023federated,
title={Federated Learning of Large Language Models with Parameter-Efficient Prompt Tuning and Adaptive Optimization},
author={Tianshi Che and Ji Liu and Yang Zhou and Jiaxiang Ren and jiwen zhou and Victor S. Sheng and Huaiyu Dai and Dejing Dou},
booktitle={The 2023 Conference on Empirical Methods in Natural Language Processing},
year={2023},
url={https://openreview.net/forum?id=WuuxbObghx}
}

@inproceedings{
weng2024probabilistic,
title={Probabilistic Federated Prompt-Tuning with Non-{IID} and Imbalanced Data},
author={Pei-Yau Weng and Minh Hoang and Lam M. Nguyen and My T. Thai and Tsui-Wei Weng and Trong Nghia Hoang},
booktitle={The Thirty-eighth Annual Conference on Neural Information Processing Systems},
year={2024},
url={https://openreview.net/forum?id=nw6ANsC66G}
}

@article{yuan2016convergence,
author = {Yuan, Kun and Ling, Qing and Yin, Wotao},
title = {On the Convergence of Decentralized Gradient Descent},
journal = {SIAM Journal on Optimization},
volume = {26},
number = {3},
pages = {1835-1854},
year = {2016},
doi = {10.1137/130943170},

URL = { 
    
        https://doi.org/10.1137/130943170
    
    

},
eprint = { 
    
        https://doi.org/10.1137/130943170
    
    

}
}

@InProceedings{cuturi2014fast,
  title = 	 {Fast Computation of Wasserstein Barycenters},
  author = 	 {Cuturi, Marco and Doucet, Arnaud},
  booktitle = 	 {Proceedings of the 31st International Conference on Machine Learning},
  pages = 	 {685--693},
  year = 	 {2014},
  editor = 	 {Xing, Eric P. and Jebara, Tony},
  volume = 	 {32},
  number =       {2},
  series = 	 {Proceedings of Machine Learning Research},
  address = 	 {Bejing, China},
  month = 	 {22--24 Jun},
  publisher =    {PMLR},
  url = 	 {https://proceedings.mlr.press/v32/cuturi14.html}
}

@article{weed2019sharp,
  title={Sharp asymptotic and finite-sample rates of convergence of empirical measures in Wasserstein distance},
  author={Jonathan Weed and Francis R. Bach},
  journal={Bernoulli},
  year={2017},
  url={https://api.semanticscholar.org/CorpusID:51919254}
}

@article{sun2022decentralized,
  title={Decentralized federated averaging},
  author={Sun, Tao and Li, Dongsheng and Wang, Bao},
  journal={IEEE Transactions on Pattern Analysis and Machine Intelligence},
  volume={45},
  number={4},
  pages={4289--4301},
  year={2022},
  publisher={IEEE}
}

@inproceedings{shi2023improving,
  title={Improving the model consistency of decentralized federated learning},
  author={Shi, Yifan and Shen, Li and Wei, Kang and Sun, Yan and Yuan, Bo and Wang, Xueqian and Tao, Dacheng},
  booktitle={International Conference on Machine Learning},
  pages={31269--31291},
  year={2023},
  organization={PMLR}
}

@inproceedings{thompson2025ntk,
  title={NTK-DFL: Enhancing Decentralized Federated Learning in Heterogeneous Settings via Neural Tangent Kernel},
  author={Thompson, Gabriel and Yue, Kai and Wong, Chau-Wai and Dai, Huaiyu},
  booktitle={International Conference on Machine Learning},
  pages={59470--59491},
  year={2025},
  organization={PMLR}
}

@inproceedings{wang2022learning,
  title={Learning to prompt for continual learning},
  author={Wang, Zifeng and Zhang, Zizhao and Lee, Chen-Yu and Zhang, Han and Sun, Ruoxi and Ren, Xiaoqi and Su, Guolong and Perot, Vincent and Dy, Jennifer and Pfister, Tomas},
  booktitle={Proceedings of the IEEE/CVF conference on computer vision and pattern recognition},
  pages={139--149},
  year={2022}
}

@ARTICLE{nedic2009distributed,
  author={Nedic, Angelia and Ozdaglar, Asuman},
  journal={IEEE Transactions on Automatic Control}, 
  title={Distributed Subgradient Methods for Multi-Agent Optimization}, 
  year={2009},
  volume={54},
  number={1},
  pages={48-61},
  doi={10.1109/TAC.2008.2009515}}

@article{ganin2016domain,
  title={Domain-adversarial training of neural networks},
  author={Ganin, Yaroslav and Ustinova, Evgeniya and Ajakan, Hana and Germain, Pascal and Larochelle, Hugo and Laviolette, Fran{\c{c}}ois and March, Mario and Lempitsky, Victor},
  journal={Journal of machine learning research},
  volume={17},
  number={59},
  pages={1--35},
  year={2016}
}

@article{darlow2018cinic,
  title={Cinic-10 is not imagenet or cifar-10},
  author={Darlow, Luke N and Crowley, Elliot J and Antoniou, Antreas and Storkey, Amos J},
  journal={arXiv preprint arXiv:1810.03505},
  year={2018}
}

@article{krizhevsky2009learning,
  title={Learning multiple layers of features from tiny images},
  author={Krizhevsky, Alex and Hinton, Geoffrey and others},
  year={2009},
  publisher={Toronto, ON, Canada}
}

@article{lecun1998mnist,
  title={The MNIST database of handwritten digits},
  author={LeCun, Yann},
  journal={http://yann. lecun. com/exdb/mnist/},
  year={1998},
}

@inproceedings{netzer2011reading,
  title={Reading digits in natural images with unsupervised feature learning},
  author={Netzer, Yuval and Wang, Tao and Coates, Adam and Bissacco, Alessandro and Wu, Baolin and Ng, Andrew Y and others},
  booktitle={NIPS workshop on deep learning and unsupervised feature learning},
  volume={2011},
  number={2},
  pages={4},
  year={2011},
  organization={Granada}
}

@misc{bulatov2011notmnist,
  title={NotMNIST Dataset},
  author={Bulatov, Yaroslav},
  year={2011},
  howpublished={\url{http://yaroslavvb.blogspot.com/2011/09/notmnist-dataset.html}}
}

@article{foret2020sharpness,
  title={Sharpness-aware minimization for efficiently improving generalization},
  author={Foret, Pierre and Kleiner, Ariel and Mobahi, Hossein and Neyshabur, Behnam},
  journal={arXiv preprint arXiv:2010.01412},
  year={2020}
}

@article{cuturi2013sinkhorn,
  title={Sinkhorn distances: Lightspeed computation of optimal transport},
  author={Cuturi, Marco},
  journal={Advances in neural information processing systems},
  volume={26},
  year={2013}
}

@book{peyre2019computational,
  title={Computational optimal transport: With applications to data science},
  author={Peyr{\'e}, Gabriel and Cuturi, Marco},
  year={2019},
  publisher={Now Foundations and Trends}
}

@article{schmitzer2019stabilized,
  title={Stabilized sparse scaling algorithms for entropy regularized transport problems},
  author={Schmitzer, Bernhard},
  journal={SIAM Journal on Scientific Computing},
  volume={41},
  number={3},
  pages={A1443--A1481},
  year={2019},
  publisher={SIAM}
}

@inproceedings{touvron2021training,
  title={Training data-efficient image transformers \& distillation through attention},
  author={Touvron, Hugo and Cord, Matthieu and Douze, Matthijs and Massa, Francisco and Sablayrolles, Alexandre and J{\'e}gou, Herv{\'e}},
  booktitle={International conference on machine learning},
  pages={10347--10357},
  year={2021},
  organization={PMLR}
}

@inproceedings{d2021convit,
  title={Convit: Improving vision transformers with soft convolutional inductive biases},
  author={d’Ascoli, St{\'e}phane and Touvron, Hugo and Leavitt, Matthew L and Morcos, Ari S and Biroli, Giulio and Sagun, Levent},
  booktitle={International conference on machine learning},
  pages={2286--2296},
  year={2021},
  organization={PMLR}
}


\twocolumn[\newpage]
\appendix
\setcounter{secnumdepth}{2}
\setcounter{assumption}{0}
\setcounter{theorem}{0}
\setcounter{lemma}{0}

\section{Additional Preliminaries}\label{appx:add_prem}

\subsection{Prompt Tuning with a Frozen Backbone}
\label{subsec:prompt_tuning}

Let $F_{\theta}$ denote a pretrained backbone model parameterized by $\theta$. In prompt tuning, the backbone parameters are kept frozen, and only a small set of prompt parameters is optimized for downstream adaptation. This parameter-efficient design substantially reduces the number of trainable parameters and makes prompt tuning attractive for decentralized learning, where communication and local computation are constrained.

For each client $u$, we denote its local dataset by $D_u$ and its prompt set at communication round $t$ by
\begin{align*}
    \omega_u^{(t)}
    =
    \big\{
    \omega_{u1}^{(t)},\dots,\omega_{un}^{(t)}
    \big\},
    \qquad
    \omega_{ui}^{(t)}\in\mathbb{R}^d,
\end{align*}
where $n$ is the number of prompts maintained by each client and $d$ is the prompt dimension. The prompts are inserted into the frozen backbone to condition the model prediction. Given an input-label pair $(x,y)\in D_u$, the client-side loss can be written as
\begin{align*}
    \ell
    \big(
    F_{\theta}(x;\omega_u^{(t)}),y
    \big),
\end{align*}
where the notation $F_{\theta}(x;\omega_u^{(t)})$ emphasizes that the prediction depends on the prompt set while the backbone parameters $\theta$ remain fixed.

At each communication round, client $u$ performs a local prompt update by optimizing only its prompt parameters:
\begin{align*}
    \omega_u^{+(t)}
    \leftarrow
    \texttt{LocalUpdate}
    \big(
    F_{\theta},D_u,\omega_u^{(t-1)}
    \big).
\end{align*}
Equivalently, this local update approximately minimizes the empirical prompt-tuning objective
\begin{align*}
    \min_{\omega_u}
    \;
    \mathcal{L}_u(\omega_u)
    :=
    \frac{1}{|D_u|}
    \sum_{(x,y)\in D_u}
    \ell
    \big(
    F_{\theta}(x;\omega_u),y
    \big),
    \qquad
    \theta \ \text{fixed}.
\end{align*}

On the other hand, the \texttt{Merge} function in decentralized prompt tuning is fundamentally different from the standard merge operation in full-model decentralized training. In full-model training, model parameters are naturally coordinate-aligned across clients, so neighboring models can often be merged by coordinate-wise weighted averaging. However, prompt sets may not have a direct index-wise correspondence across clients. Since each client optimizes prompts using its own local data distribution, the learned prompts may drift toward client-specific directions. Therefore, a naive prompt-level \texttt{Merge} function that averages prompts by index may combine semantically misaligned prompts and produce less representative local states. This motivates a specialized \texttt{Merge} function that aligns and summarizes exchanged neighborhood prompts before producing the updated local prompt set.

\subsection{Decentralized Federated Learning Baselines}\label{appendix:dfl_baseline}

We compare D-FROST against three representative DFL methods, each adapted to the prompt-tuning protocol.
Under this protocol, the pretrained ViT-B/32 backbone is kept frozen throughout training, and only the learnable prompt tokens and the task-specific classification head are updated locally and exchanged during each communication round.
We refer to the adapted versions as \textbf{D-PSGD-PT}, \textbf{DFedAvgM-PT}, and \textbf{DFedSAM-PT}, respectively.

In all three baselines, the \texttt{Merge} step is implemented as coordinate-wise weighted averaging through the doubly stochastic mixing matrix $W$.
Concretely, after local updates, each client $u$ receives the updated prompt parameters $\tilde{\omega}_v^{(t)}$ from each neighbor $v \in \mathcal{N}(u)$ and sets
\begin{equation}
    \omega_u^{(t)} \leftarrow \sum_{v \in \mathcal{N}(u) \cup \{u\}} W_{uv}\, \tilde{\omega}_v^{(t)}.
\end{equation}
This index-wise average assumes that prompt tokens at the same position index across clients represent semantically comparable directions, an assumption that fails under heterogeneous data, where locally optimized prompts are free to drift into client-specific subspaces.

\paragraph{D-PSGD~\citep{lian2017can}.}
D-PSGD is a classic decentralized parallel SGD method that uses one-step SGD to
train local models in each communication round. Each client performs one local
mini-batch update with plain SGD, followed by a neighbor-averaging \texttt{Merge}
step through $W$. Following the standard protocol, the training epoch in D-PSGD
is set to $1$, whereas it is set to $5$ for all other baselines and D-FROST, so
D-PSGD performs strictly less local computation per round.

\paragraph{DFedAvgM~\citep{sun2022decentralized}.}
DFedAvgM extends FedAvg to the decentralized setting by allowing clients to perform multiple local SGD iterations with momentum before communicating, reducing the number of communication rounds needed for convergence compared to D-PSGD.
After local training, the \texttt{Merge} step applies index-wise weighted averaging through $W$.

\paragraph{DFedSAM~\citep{shi2023improving}.}
DFedSAM improves upon DFedAvgM by replacing local SGD with Sharpness-Aware Minimization (SAM)~\citep{foret2020sharpness}, which seeks parameters that lie in flat loss neighborhoods, thereby reducing the inconsistency that arises among local models trained on heterogeneous data.
Each local update consists of a two-step SAM procedure: a perturbation step that moves parameters toward the neighborhood of highest loss, followed by a gradient step evaluated at the perturbed point.
As with the other baselines, the \texttt{Merge} step performs index-wise weighted averaging through $W$.\newline

Beyond these baselines, we also discuss several other related FL protocols.

\paragraph{NTK-DFL~\citep{thompson2025ntk}.}
NTK-DFL replaces stochastic gradient updates with Neural Tangent Kernel-based weight
evolution to improve convergence under heterogeneous data, achieving $4.6\times$ fewer
communication rounds than DFedAvgM on Fashion-MNIST.
However, the NTK linearization is currently restricted to small two-layer MLPs and the
authors explicitly acknowledge that scaling to CNNs or Transformers remains an open
problem.
\paragraph{PFPT~\citep{weng2024probabilistic}.}
PFPT is a centralized federated prompt-tuning method that addresses prompt misalignment under non-IID and imbalanced data by treating prompt aggregation as a distributed set modeling problem, with a dynamically sized global pool maintained by a central server. The pool expands to accommodate new, semantically distinct prompts contributed by clients with heterogeneous local distributions, while prompts that are sufficiently similar are merged to suppress redundancy and keep the prompt pool compact. In the centralized FL setting, the prompt pool slightly expands in early rounds then stabilizes as the server merges semantically similar prompts from all clients.

However, extending this idea to decentralized communication
is nontrivial, since each client only observes local neighbor-
hood prompts rather than a global collection. When this method is naively extended to DFL by replacing the central server with neighborhood-level aggregation,  local prompt pools grow exponentially across rounds, inflating local model size rapidly (Figure~\ref{fig:decentralized_pfpt_explosion}). As each client merges only with its neighbors' prompts, and under high data heterogeneity these neighborhood pools share little semantic overlap, so their union rarely contracts. The problem compounds because each client starts every round from its own diverged local pool rather than a shared global one as in PFPT.
\begin{figure}
    \centering
    \includegraphics[width=\linewidth]{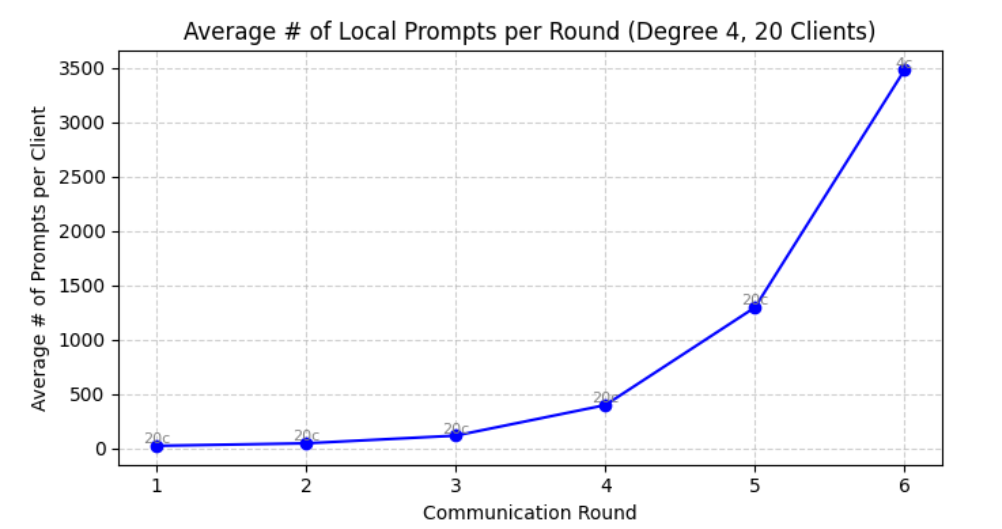}
    \caption{Average number of local prompts per client in decentralized PFPT ($\kappa = 4$, 20 clients, Fashion-MNIST). 
    The number of prompts increases exponentially across rounds.}
    \label{fig:decentralized_pfpt_explosion}
\end{figure}

\section{Theoretical Analysis of \NAMEA\ (More Details)}
\label{appendix:theoretical_analysis}

In this section, we analyze the theoretical properties of \NAMEA. The analysis is organized around the two key components of the algorithm. We first study the local OT-based prompt merging problem and show that its alternating solver is stable. We then analyze the global behavior of \NAMEA\ in the decentralized network. By viewing each client prompt set as an empirical measure in the Wasserstein space, we establish that the network-level prompt barycenter converges to a neighborhood of Wasserstein stationarity for the shared prompt-tuning objective in \eqref{eq:global_prompt_measure_objective}.

\subsection{Stability of the Local OT-Based Merge}
\label{subsec:local_convergence}

We first analyze the local OT-based \texttt{Merge} operator used in \NAMEA. Recall that after client $u$ forms the neighborhood prompt collection $\Omega_u^{(t)}$, it solves the local OT problem in \eqref{eq:full_objective} by alternating between the transport update and the barycenter update. Since the merged prompt set is used as the client state for the next communication round, the inner OT solver should produce a stable representative prompt set.
The main result of this subsection shows that the alternating OT solver becomes stable as the number of inner steps increases. 

In Lemma~\ref{lem:lower_bound}, we first establish a lower bound on the local OT objective, which is used to control the total objective decrease across the inner iterations.

\begin{lemma}[Lower Bound of the Local OT Objective]
\label{lem:lower_bound}
For any feasible transport plan $P$ satisfying $P\mathbf{1}_n=q$ and any representative prompt set $\Phi$, the local OT merge objective $\mathcal{J}_u(P,\Phi)$ in \eqref{eq:full_objective} is bounded from below:
\begin{align*}
    \mathcal{J}_u(P,\Phi)
    \ge
    \mathcal{J}^*
    :=
    -\varepsilon
    \big(
    \log(nN_u^{(t)})+1
    \big).
\end{align*}
\end{lemma}

\begin{proof}
The spatial matching term is non-negative because
\begin{align*}
    C_{ai}(\Phi)
    =
    \frac{1}{2\sigma^2}
    \|z_a-\phi_i\|^2
    \ge 0,
\end{align*}
and therefore $\langle P,C(\Phi)\rangle\ge 0$. The $L_2$ regularization term is also non-negative:
\begin{align*}
    \frac{\lambda}{2\sigma^2}
    \sum_{i=1}^{n}
    \|\phi_i\|^2
    \ge 0.
\end{align*}

It remains to lower-bound the entropy term. Since $P\mathbf{1}_n=q$ and $q=\frac{1}{N_u^{(t)}}\mathbf{1}_{N_u^{(t)}}$, we have
\begin{align*}
    \sum_{a=1}^{N_u^{(t)}}\sum_{i=1}^{n}P_{ai}=1.
\end{align*}
The quantity $\sum_{a,i}P_{ai}\log P_{ai}$ is minimized over the probability simplex when the mass is uniformly distributed, i.e.,
$
    P_{ai}
    =
    \frac{1}{nN_u^{(t)}},
    \qquad
    \forall a,i.
$.
Thus, we have:
\begin{align*}
    \varepsilon
    \sum_{a=1}^{N_u^{(t)}}
    \sum_{i=1}^{n}
    P_{ai}(\log P_{ai}-1)
    &\ge
    \varepsilon
    \sum_{a=1}^{N_u^{(t)}}
    \sum_{i=1}^{n}
    \frac{1}{nN_u^{(t)}}
    \log
    \frac{1}{nN_u^{(t)}}
    \\ &-
    \varepsilon
    \sum_{a=1}^{N_u^{(t)}}
    \sum_{i=1}^{n}
    P_{ai}  \\
    &=
    -\varepsilon\log(nN_u^{(t)})-\varepsilon  \\
    &=
    -\varepsilon
    \big(
    \log(nN_u^{(t)})+1
    \big).
\end{align*}
Combining this entropy lower bound with the non-negativity of the spatial matching and $L_2$ terms gives the claimed lower bound.
\end{proof}

We now use Lemma~\ref{lem:lower_bound} to show that the alternating solver stabilizes. Let
\begin{align*}
    \Delta^{(s)}
    :=
    \|P^{(s)}-P^{(s-1)}\|_F
    +
    \|\Phi^{(s)}-\Phi^{(s-1)}\|_F
\end{align*}
denote the discrepancy between two consecutive inner iterations. We have Theorem~\ref{thm:convergence_rate} as follows:

\begin{theorem}[Stability of the Alternating OT Solver]
\label{thm:convergence_rate}
Let $\Delta^{(s)}$ be defined as above, and let
$
    \mu
    :=
    \min
    \left(
    \varepsilon,
    \frac{\lambda}{\sigma^2}
    \right)
    >
    0.
$
After $S$ alternating OT steps, the minimum iterate difference is bounded by an $\mathcal{O}(1/\sqrt{S})$ rate:
\begin{align*}
    \min_{1\le s\le S}
    \Delta^{(s)}
    \le
    \sqrt{
    \frac{
    4
    \big(
    \mathcal{J}_u(P^{(0)},\Phi^{(0)})
    -
    \mathcal{J}^*
    \big)
    }{
    \mu S
    }
    }.
\end{align*}
\end{theorem}

\begin{proof}
The alternating solver consists of two exact block minimization steps. We first consider the transport step. With $\Phi^{(s-1)}$ fixed, the transport subproblem is
\begin{align*}
    \min_{P\ge 0}
    \left\{
    \langle P,C(\Phi^{(s-1)})\rangle
    +
    \varepsilon
    \sum_{a=1}^{N_u^{(t)}}
    \sum_{i=1}^{n}
    P_{ai}(\log P_{ai}-1)
    \right\}
    \\
    \text{s.t.}
    \quad
    P\mathbf{1}_n=q.
\end{align*}
The entropy term makes this subproblem strongly convex in $P$. Therefore, since $P^{(s)}$ is the exact minimizer, we obtain
{\small
\begin{align}
    \mathcal{J}_u(P^{(s-1)},\Phi^{(s-1)})
    -
    \mathcal{J}_u(P^{(s)},\Phi^{(s-1)})
    \ge
    \frac{\varepsilon}{2}
    \|P^{(s)}-P^{(s-1)}\|_F^2.
    \label{eq:transport_descent}
\end{align}
}
Next, consider the barycenter step. With $P^{(s)}$ fixed, the representative prompt update solves
\begin{align*}
    \min_{\Phi}
    \left\{
    \langle P^{(s)},C(\Phi)\rangle
    +
    \frac{\lambda}{2\sigma^2}
    \sum_{i=1}^{n}
    \|\phi_i\|^2
    \right\}.
\end{align*}
For each representative prompt $\phi_i$, the Hessian of this subproblem is
\begin{align*}
    \frac{1}{\sigma^2}
    \left(
    \sum_{a=1}^{N_u^{(t)}}P_{ai}^{(s)}
    +
    \lambda
    \right)I.
\end{align*}
Since $P_{ai}^{(s)}\ge 0$, the minimum eigenvalue is at least $\lambda/\sigma^2$. Hence, the barycenter subproblem is $\lambda/\sigma^2$-strongly convex. Since $\Phi^{(s)}$ is the exact minimizer, we have
\begin{align}
    \mathcal{J}_u(P^{(s)},\Phi^{(s-1)})
    -
    \mathcal{J}_u(P^{(s)},\Phi^{(s)})
    \ge
    \frac{\lambda}{2\sigma^2}
    \|\Phi^{(s)}-\Phi^{(s-1)}\|_F^2.
    \label{eq:barycenter_descent}
\end{align}

Combining \eqref{eq:transport_descent} and \eqref{eq:barycenter_descent}, and defining
$
    \mu
    =
    \min
    \left(
    \varepsilon,
    \frac{\lambda}{\sigma^2}
    \right),
$
we obtain
\begin{align*}
    &\mathcal{J}_u(P^{(s-1)},\Phi^{(s-1)})
    -
    \mathcal{J}_u(P^{(s)},\Phi^{(s)})\\
    &\ge
    \frac{\mu}{2}
    \left(
    \|P^{(s)}-P^{(s-1)}\|_F^2
    +
    \|\Phi^{(s)}-\Phi^{(s-1)}\|_F^2
    \right)  \\
    &\ge
    \frac{\mu}{4}
    \left(
    \|P^{(s)}-P^{(s-1)}\|_F
    +
    \|\Phi^{(s)}-\Phi^{(s-1)}\|_F
    \right)^2  \\
    &=
    \frac{\mu}{4}
    \big(\Delta^{(s)}\big)^2.
\end{align*}

Summing over $s=1,\dots,S$ gives
\begin{align*}
    \sum_{s=1}^{S}
    \big(\Delta^{(s)}\big)^2
    &\le
    \frac{4}{\mu}
    \sum_{s=1}^{S}
    \left[
    \mathcal{J}_u(P^{(s-1)},\Phi^{(s-1)})
    -
    \mathcal{J}_u(P^{(s)},\Phi^{(s)})
    \right]  \\
    &=
    \frac{4}{\mu}
    \left[
    \mathcal{J}_u(P^{(0)},\Phi^{(0)})
    -
    \mathcal{J}_u(P^{(S)},\Phi^{(S)})
    \right].
\end{align*}
By Lemma~\ref{lem:lower_bound}, $\mathcal{J}_u(P^{(S)},\Phi^{(S)})\ge \mathcal{J}^*$. Therefore,
\begin{align*}
    \sum_{s=1}^{S}
    \big(\Delta^{(s)}\big)^2
    \le
    \frac{4}{\mu}
    \left[
    \mathcal{J}_u(P^{(0)},\Phi^{(0)})
    -
    \mathcal{J}^*
    \right].
\end{align*}
Finally, since we have
\begin{align*}
    S\min_{1\le s\le S}
    \big(\Delta^{(s)}\big)^2
    \le
    \sum_{s=1}^{S}
    \big(\Delta^{(s)}\big)^2,
\end{align*}
we obtain
\begin{align*}
    \min_{1\le s\le S}
    \Delta^{(s)}
    \le
    \sqrt{
    \frac{
    4
    \big(
    \mathcal{J}_u(P^{(0)},\Phi^{(0)})
    -
    \mathcal{J}^*
    \big)
    }{
    \mu S
    }
    }.
\end{align*}
\end{proof}

Theorem~\ref{thm:convergence_rate} provides a quantitative stability guarantee for the local OT-based \texttt{Merge} step. It shows that, among the first $S$ alternating iterations, there exists at least one iterate whose change from the previous iterate is bounded by $\mathcal{O}(1/\sqrt{S})$. Therefore, increasing the number of inner OT steps makes the local merge solution progressively more stable. The bound also makes explicit how the stability depends on the initial objective gap $\mathcal{J}_u(P^{(0)},\Phi^{(0)})-\mathcal{J}^*$ and the effective strong-convexity parameter $\mu=\min(\varepsilon,\lambda/\sigma^2)$.

This result has two implications for \NAMEA. First, the OT-based \texttt{Merge} step does not behave as an uncontrolled heuristic. Specifically, its alternating updates have a provable descent structure and converge toward a stable local representative prompt set. Second, the number of inner steps $S$ controls the quality of the local merge output. A larger $S$ reduces the inner solver instability, which in turn reduces the approximation error introduced when replacing the neighborhood prompt collection by the compact representative set $\Phi^{(S)}$. This local control is the basis for the subsequent network-level analysis, where the error of the OT-based merge affects Wasserstein consensus and the convergence of the network barycenter.

\subsection{Global Wasserstein Consensus and Stationarity}
\label{subsec:global_convergence}

We now analyze the global behavior of \NAMEA. The analysis establishes two main results. First, the local prompt measures remain close to a network-level barycenter, showing that \NAMEA\ controls the Wasserstein consensus error across clients. Second, this network-level barycenter makes progress toward stationarity of the shared prompt-tuning objective in \eqref{eq:global_prompt_measure_objective}.

First of all, we define notions needed for our analysis. For each client $u$, we represent its prompt set at round $t$ as the empirical measure
\begin{equation}
    \mu_u^{(t)}
    :=
    \frac{1}{n}
    \sum_{i=1}^{n}
    \delta_{\omega_{ui}^{(t)}}.
\end{equation}
Let $\mu_{\mathrm{avg}}^{(t)}$ denote the Wasserstein barycenter of the local prompt measures, as defined in \eqref{eq:wasserstein_network_barycenter}.

\paragraph{Local Update.}
During local training, client $u$ updates its prompts via standard backpropagation. In measure space, this is strictly equivalent to updating the discrete empirical measure via the push-forward of the Wasserstein gradient:
\begin{equation}
    \mu_u^{+(t)} = \big( I - \eta \nabla_{W_2} \mathcal{F}_u \big)_{\#} \mu_u^{(t-1)}.
\end{equation}

\paragraph{Ideal Reference States.}
During communication, the network seeks consensus over the doubly stochastic graph topology $W$. If communication were exact, the neighborhood would converge to the Wasserstein barycenter. We define these theoretical targets for both the local neighborhood and the global network:
\begin{enumerate}
    \item The \emph{ideal local barycenter} $\nu_u^{(t)}$:
    \begin{equation}
        \nu_u^{(t)} := \arg\min_{\nu \in \mathcal{P}_2(\mathbb{R}^d)} \sum_{v=1}^m W_{uv} W_2^2\big(\nu, \mu_v^{+(t)}\big).
    \end{equation}
    \item The \emph{ideal global state} $\nu_{\mathrm{avg}}^{(t)}$, tracking the exact network center of mass:
    \begin{equation}
        \nu_{\mathrm{avg}}^{(t)} := \arg\min_{\nu \in \mathcal{P}_2(\mathbb{R}^d)} \frac{1}{m} \sum_{u=1}^m W_2^2\big(\nu, \mu_u^{+(t)}\big).
    \end{equation}
\end{enumerate}

\paragraph{Practical Aggregation and Consensus Error.}
Because computing the exact barycenters causes the prompt support size to grow indefinitely, Algorithm~\ref{alg:dec_ot} applies a fixed-budget optimal transport estimator. We denote this optimal transport compression operator as $\mathcal{C}_{\text{OT}}^n$, which projects the neighborhood updates onto a strict $n$-point summarizing measure:
\begin{equation}
    \mu_u^{(t)} := \mathcal{C}_{\text{OT}}^n \Big( \big\{ \mu_v^{+(t)} \big\}_{v \in \mathcal{N}(u) \cup \{u\}} \Big).
\end{equation}
Consequently, the \emph{actual global state} of the network is simply the barycenter of these compressed measures:
\begin{equation}
    \mu_{\mathrm{avg}}^{(t)} := \arg\min_{\mu \in \mathcal{P}_2(\mathbb{R}^d)} \frac{1}{m} \sum_{u=1}^m W_2^2\big(\mu, \mu_u^{(t)}\big).
\end{equation}

To track the convergence of the network, we define the \textbf{network consensus error} at round $t$ as the average squared 2-Wasserstein distance between the individual clients' compressed states and the actual global state:
\begin{equation}
    \varepsilon^{(t)} := \frac{1}{m} \sum_{u=1}^m W_2^2\big(\mu_u^{(t)}, \mu_{\mathrm{avg}}^{(t)}\big).
    \label{eq:appendix:consensus_error_def}
\end{equation}

\paragraph{Optimal Transport Displacement and Bounded Compression.}
To rigorously isolate the algorithmic distortion of $\mathcal{C}_{\text{OT}}^n$ and track the network's optimization trajectory, we must geometrically map the movement of these measures. In the 2-Wasserstein space, the displacement between two probability measures is defined by the vector field that optimally transports one measure into the other. For any two absolutely continuous measures $\mu, \nu \in \mathcal{P}_2(\mathbb{R}^d)$, let $T_{\mu \to \nu}$ be the optimal transport map. The optimal transport displacement is the vector field $v \in L^2(\mu; \mathbb{R}^d)$ defined as $v(x) := T_{\mu \to \nu}(x) - x$. We denote this mapping as the inverse exponential map:
\begin{equation}
    v = \text{exp}_{\mu}^{-1}(\nu).
\end{equation}
By definition, the squared $L^2(\mu)$ norm of this vector field equals the squared 2-Wasserstein distance: $\|v\|_{L^2(\mu)}^2 = W_2^2(\mu, \nu)$. 

Let $\tilde{v}^{(t)} = \text{exp}_{\mu_{\mathrm{avg}}^{(t-1)}}^{-1}(\nu_{\mathrm{avg}}^{(t)})$ be the ideal displacement mapping to the uncompressed global barycenter, and $v^{(t)} = \text{exp}_{\mu_{\mathrm{avg}}^{(t-1)}}^{-1}(\mu_{\mathrm{avg}}^{(t)})$ be the actual displacement mapping to the OT-compressed global barycenter. By comparing these two vector fields, the expected projection error of the $\mathcal{C}_{\text{OT}}^n$ compressor is bounded by:
\begin{equation}
    \mathbb{E}\left[ \left\| \tilde{v}^{(t)} - v^{(t)} \right\|_{L^2(\mu_{\mathrm{avg}}^{(t-1)})}^2 \right] \le \delta_n^2(S),
    \label{eq:ot_compression_bound}
\end{equation}
where $S$ is the number of steps $\mathcal{C}_{\text{OT}}^n$ is allowed to run. Crucially, established theoretical results on the finite-sample approximation of Wasserstein barycenters \cite{cuturi2014fast} and the convergence rates of empirical measures \cite{weed2019sharp} guarantee that this error is bounded above. The approximation error $\delta_n^2(S)$ monotonically decays as the client prompt budget $n$ and the number of iterative solver steps $S$ increase.

To complete the convergence framework, we introduce the standard assumptions.

\begin{assumption}[Wasserstein Smoothness]
\label{assum:w2_smooth_rev}
The local loss functional $\mathcal{F}_u(\mu)$ is $L$-smooth over the 2-Wasserstein space. Consequently, the global functional $\mathcal{F}(\mu)$ is also $L$-smooth and bounded below by $\mathcal{F}^* > -\infty$.
\end{assumption}

By the mathematical definition of $L$-smoothness in the 2-Wasserstein space, Assumption~\ref{assum:w2_smooth_rev} guarantees that for any two absolutely continuous probability measures $\mu$ and $\nu$ connected by the optimal transport displacement $v = \text{exp}_{\mu}^{-1}(\nu)$, the functional satisfies the Taylor-type upper bound:
\begin{equation}
    \mathcal{F}(\nu) \le \mathcal{F}(\mu) + \langle \nabla_{W_2} \mathcal{F}(\mu), v \rangle_{L^2(\mu)} + \frac{L}{2} \|v\|_{L^2(\mu)}^2.
    \label{eq:w2_smoothness_bound}
\end{equation}

\begin{assumption}[Bounded Variance]
\label{assum:ot_variance_error}
The variance of the local Wasserstein gradients is uniformly bounded: $\mathbb{E}[\|\nabla_{W_2} \mathcal{F}_u(\mu)\|_{L^2}^2] \le G^2$.
\end{assumption}

\begin{assumption}[Bounded Prompt Support]
\label{assum:bounded_support}
The support of the prompt distributions remains within a bounded domain. Specifically, there exists a constant $D > 0$ such that for all prompts $\omega_{ui}^{(t)}$, $\|\omega_{ui}^{(t)}\| \le D$.
\end{assumption}

\begin{assumption}[Graph Spectral Gap]
\label{assum:graph_mixing}
The communication matrix $W \in \mathbb{R}^{m \times m}$ is symmetric and doubly stochastic. Its second largest eigenvalue magnitude governs the spectral gap, defining the network contraction factor:
\begin{equation}
    \rho := \left\| W - \frac{1}{m}\mathbf{1}\mathbf{1}^{\top} \right\|_2 < 1.
\end{equation}
\end{assumption}

Assumption~\ref{assum:graph_mixing} dictates the standard network topology conditions in decentralized optimization literature \cite{nedic2009distributed, lian2017can}. The doubly stochastic property guarantees that the exact global average of the network is strictly preserved during the gossip step. The symmetry of $W$ implies bidirectional communication channels with equal weightings. Finally, the spectral gap condition $\rho < 1$ is algebraically equivalent to assuming the underlying communication graph is connected and non-bipartite. This geometric property ensures that information from any isolated client will eventually propagate to all other clients, providing the mathematical engine that drives the linear contraction of local states toward the global mean \cite{koloskova2020unified}.

\textbf{Step 1. Network Consensus in the Wasserstein Space.} 
Before we can establish the final optimization convergence rate of the decentralized algorithm, we must first prove that the network successfully reaches a state of geometric consensus. The central theoretical challenge is that local prompt-tuning pulls the clients' distributions apart, while the graph communication and Optimal Transport (OT) compression attempt to pull them together. 

To rigorously bound this dynamic, we decompose the network's behavior into three fundamental mechanics:
\begin{enumerate}
    \item \textbf{Global Average Preservation (Lemma~\ref{lem:global_average}):} We prove that the doubly stochastic graph topology strictly preserves the exact center of mass of the network.
    \item \textbf{Local Dispersion (Lemma~\ref{lem:local_dispersion}):} We bound how far the local gradient updates drag the clients away from this global center of mass.
    \item \textbf{Graph Contraction (Lemma~\ref{lem:graph_contraction}):} We map the distributions into a flat kernel space to prove that the communication step strictly contracts this dispersion by the graph's spectral gap.
\end{enumerate}
By combining these three mechanics, we construct a linear recurrence relation that permanently traps the network consensus error $\varepsilon^{(t)}$ within a bounded mathematical neighborhood.

\begin{lemma}[Preservation of the Global Average]
\label{lem:global_average}
Let $\mu_{\mathrm{avg}}^{+(t)} := \frac{1}{m} \sum_{u=1}^{m} \mu_u^{+(t)}$ be the average of the locally updated states. During the gossip communication step, the ideal continuous barycenter of the network exactly equals this updated average:
\begin{equation}
    \nu_{\mathrm{avg}}^{(t)} = \mu_{\mathrm{avg}}^{+(t)} \qquad \forall t.
\end{equation}
\end{lemma}

\begin{proof}
By expanding the definition of the ideal global barycenter $\nu_{\mathrm{avg}}^{(t)}$ and exchanging the order of summation, we obtain:
\begin{align*}
    \nu_{\mathrm{avg}}^{(t)} 
    &= \frac{1}{m} \sum_{u=1}^{m} \nu_u^{(t)} 
    = \frac{1}{m} \sum_{u=1}^{m} \sum_{v=1}^{m} W_{uv}\mu_v^{+(t)} 
    \\ &= \frac{1}{m} \sum_{v=1}^{m} \mu_v^{+(t)} \left( \sum_{u=1}^{m} W_{uv} \right).
\end{align*}
Because the communication matrix $W$ is column-stochastic (Assumption~\ref{assum:graph_mixing}), the inner sum strictly equals $1$ for all $v$. The expression immediately simplifies to $\mu_{\mathrm{avg}}^{+(t)}$.
\end{proof}

\begin{lemma}[Local Dispersion Bound]
\label{lem:local_dispersion}
Under the $L$-Lipschitz smoothness and bounded gradient variance ($G^2$) assumptions, the geometric dispersion of the locally updated states from their global average is bounded by the previous consensus error $\varepsilon^{(t-1)}$:
\begin{equation}
    \frac{1}{m} \sum_{u=1}^{m} W_2^2 \left( \mu_u^{+(t)}, \mu_{\mathrm{avg}}^{+(t)} \right) 
    \le 4(1+\eta^2L^2)\varepsilon^{(t-1)} + 4\eta^2G^2.
\end{equation}
\end{lemma}

\begin{proof}
We introduce an intermediate virtual state, $\mu_{\mathrm{mid}}^{(t)} := \big( I-\eta\nabla_{W_2}\mathcal{F}(\mu_{\mathrm{avg}}^{(t-1)}) \big)_{\#} \mu_{\mathrm{avg}}^{(t-1)}$, which represents a perfectly synchronized gradient step. Applying the relaxed triangle inequality $W_2^2(a,c) \le 2W_2^2(a,b) + 2W_2^2(b,c)$ and averaging over $m$ clients, we have:
\begin{align*}
    \frac{1}{m} \sum_{u=1}^{m} W_2^2 \left( \mu_u^{+(t)}, \mu_{\mathrm{avg}}^{+(t)} \right) 
    &\le 2 W_2^2 \left( \mu_{\mathrm{mid}}^{(t)}, \mu_{\mathrm{avg}}^{+(t)} \right) \\
    &+ \frac{2}{m} \sum_{u=1}^{m} W_2^2 \left( \mu_u^{+(t)}, \mu_{\mathrm{mid}}^{(t)} \right).
\end{align*}
For the second term, we bound the distance between the local push-forward map and the synchronized push-forward map. By adding and subtracting the local gradients evaluated at the global average, and utilizing the $L$-smoothness and variance bounds, the $L^2$ mapping error is strictly bounded by $(1+\eta^2L^2)\varepsilon^{(t-1)} + \eta^2G^2$. 

Applying a similar push-forward expansion to the first term via Jensen's inequality isolates the gradient deviations across the network. Summing the symmetric bounds together absorbs the remaining distances.
\end{proof}

\begin{lemma}[Graph Contraction via MMD Equivalence]
\label{lem:graph_contraction}
Let the prompt distributions satisfy the bounded support constraint $D$ (Assumption~\ref{assum:bounded_support}). The gossip communication step strictly contracts the network dispersion by the graph's spectral gap $\rho^2$:
\begin{equation}
    \frac{1}{m} \sum_{u=1}^{m} W_2^2 \left( \nu_u^{(t)}, \nu_{\mathrm{avg}}^{(t)} \right) 
    \le \mathcal{D}\rho^2 \frac{1}{m} \sum_{u=1}^{m} W_2^2 \left( \mu_u^{+(t)}, \mu_{\mathrm{avg}}^{+(t)} \right),
\end{equation}
where $\mathcal{D} > 0$ is a metric translation constant.
\end{lemma}

\begin{proof}
Because Wasserstein space is non-linear, we map the empirical measures into a Reproducing Kernel Hilbert Space (RKHS) using the kernel mean embedding $\Phi(\mu)$. Let $\theta_u := \Phi(\mu_u^{+(t)})$ and $y_u := \Phi(\nu_u^{(t)})$. Because the embedding is linear, the graph mixing applies directly to the RKHS vectors: $y_u = \sum_v W_{uv} x_v$. By defining the mean-centered vectors $\bar{x}_u$ and $\bar{y}_u$, standard Euclidean algebraic graph theory provides the spectral bound:
{\small
\begin{align*}
    \sum_{u=1}^{m} \|\bar{y}_u\|_{\mathcal{H}_k}^2 
    \le \left\| W - \frac{1}{m}\mathbf{1}\mathbf{1}^\top \right\|_{\mathrm{op}}^2 \sum_{u=1}^{m} \|\bar{x}_u\|_{\mathcal{H}_k}^2 
    = \rho^2 \sum_{u=1}^{m} \|\bar{x}_u\|_{\mathcal{H}_k}^2.
\end{align*}
}
Because the RKHS norm is precisely the Maximum Mean Discrepancy (MMD), this establishes the contraction in MMD: $\sum_u \mathrm{MMD}^2(\nu_u^{(t)}, \nu_{\mathrm{avg}}^{(t)}) \le \rho^2 \sum_u \mathrm{MMD}^2(\mu_u^{+(t)}, \mu_{\mathrm{avg}}^{+(t)})$. 

Under the bounded support constraint, MMD and 2-Wasserstein metrics are topologically equivalent. Consequently, there exist strict constants $C_D^1, C_D^2 > 0$ such that $C_D^1 W_2^2 \le \mathrm{MMD}^2 \le C_D^2 W_2^2$. Dividing these boundary constraints yields the metric translation constant $\mathcal{D} = C_D^2 / C_D^1$, converting the RKHS contraction back into the Wasserstein space.
\end{proof}

\begin{theorem}[Rigorous Wasserstein Consensus Bound]
\label{thm:global_consensus_bound}
Under the assumptions of bounded gradients, bounded prompt support, and a doubly stochastic mixing matrix, the expected network consensus error $\varepsilon^{(t)}$ converges asymptotically to a stationary bounded neighborhood. Specifically, for any $t \to \infty$:
\begin{equation}
    \varepsilon^{(t)} \le \frac{\beta}{1-\alpha} = \mathcal{O}\left( \frac{\delta_n^2(S)}{1 - \mathcal{D}\rho^2} + \frac{\eta^2 G^2}{1 - \mathcal{D}\rho^2} \right),
\end{equation}
where $\alpha = 6\mathcal{D}\rho^2 (1+\eta^2L^2) < 1$, and $\beta = 6\delta_n^2(S) + 6\mathcal{D}\rho^2\eta^2G^2$.
\end{theorem}

\begin{proof}
We expand the consensus error $\varepsilon^{(t)}$ by chaining the discrete mapping steps through the relaxed three-way triangle inequality $(a+b+c)^2 \le 3(a^2+b^2+c^2)$:
\begin{align*}
    W_2^2 \big(\mu_u^{(t)}, \mu_{\mathrm{avg}}^{(t)}\big) 
    &\le 3 W_2^2\big(\mu_u^{(t)}, \nu_u^{(t)}\big) 
    + 3 W_2^2\big(\nu_u^{(t)}, \nu_{\mathrm{avg}}^{(t)}\big) 
    \\ &+ 3 W_2^2\big(\nu_{\mathrm{avg}}^{(t)}, \mu_{\mathrm{avg}}^{(t)}\big).
\end{align*}

We bound each of the three segments by averaging over all $m$ clients. 
First, the term $\frac{1}{m} \sum_u W_2^2(\mu_u^{(t)}, \nu_u^{(t)})$ explicitly represents the projection error of the OT compressor, which is bounded by $\delta_n^2(S)$. Second, by the joint convexity of the Wasserstein metric, the divergence between the actual compressed global average and the ideal global average, $W_2^2(\mu_{\mathrm{avg}}^{(t)}, \nu_{\mathrm{avg}}^{(t)})$, is identically bounded by the average of the local compression errors, yielding another $\delta_n^2(S)$.

For the central graph tracking term, we sequentially apply the contraction bound from Lemma~\ref{lem:graph_contraction} and the dispersion bound from Lemma~\ref{lem:local_dispersion}:
{\small
\begin{align*}
    \frac{1}{m} \sum_{u=1}^{m} W_2^2 \left( \nu_u^{(t)}, \nu_{\mathrm{avg}}^{(t)} \right) 
    \le 2\mathcal{D}\rho^2 \left( (1+\eta^2L^2)\varepsilon^{(t-1)} + \eta^2G^2 \right).
\end{align*}
}
Plugging these three bounds back into the triangle expansion collapses the dynamics into a single linear recurrence relation:
\begin{align*}
    \varepsilon^{(t)} \le \underbrace{\left[ 6\mathcal{D}\rho^2 (1+\eta^2L^2) \right]}_{:= \alpha} \varepsilon^{(t-1)} + \underbrace{\left[ 6\delta_n^2(S) + 6\mathcal{D}\rho^2 \eta^2 G^2 \right]}_{:= \beta}.
\end{align*}
To strictly ensure geometric convergence ($\alpha < 1$), we require a learning rate satisfying $\eta^2 < \frac{1}{L^2} \big( \frac{1}{6\mathcal{D}\rho^2} - 1 \big)$. Unrolling the recurrence relation as $t \to \infty$ yields the infinite geometric series bound $\beta / (1-\alpha)$.
\end{proof}

\textbf{Step 2. Global Optimization Convergence.} 
With the network geometrically trapped in a tight consensus neighborhood, we can now bound the deviation of the network's gradient trajectory from the ideal centralized trajectory. 

\begin{lemma}[Tangent-Space Tracking Error]
\label{lem:tracking_error}
Let $\hat{v}^{(t)} = -\eta \nabla_{W_2} \mathcal{F}(\mu_{\mathrm{avg}}^{(t-1)})$ be the virtual global gradient displacement, and let $v^{(t)} = \text{exp}_{\mu_{\mathrm{avg}}^{(t-1)}}^{-1}(\mu_{\mathrm{avg}}^{(t)})$ be the actual optimal transport map to the true network barycenter. Under Assumption~\ref{assum:w2_smooth_rev}, the expected tangent-space tracking error is strictly bounded by:
{\small
\begin{align}
    \mathbb{E} \left\| \hat{v}^{(t)} - v^{(t)} \right\|_{L^2(\mu_{\mathrm{avg}}^{(t-1)})}^2 
    &\le 
    \frac{2 \eta^2 L^2}{m} \sum_{u=1}^m \mathbb{E} \left[ W_2^2(\mu_{\mathrm{avg}}^{(t-1)}, \mu_u^{(t-1)}) \right] \nonumber
    \\&+ 2\delta_n^2(S).
    \label{eq:tracking_error_bound}
\end{align}
}
\end{lemma}

\begin{proof}
We bound the divergence between the virtual map and the actual map by introducing the intermediate ideal displacement field $\tilde{v}^{(t)} = \text{exp}_{\mu_{\mathrm{avg}}^{(t-1)}}^{-1}(\nu_{\mathrm{avg}}^{(t)})$, which maps to the uncompressed global barycenter. Applying the relaxed triangle inequality in the Hilbert space $L^2(\mu_{\mathrm{avg}}^{(t-1)})$, we have:
\begin{equation}
    \left\| \hat{v}^{(t)} - v^{(t)} \right\|_{L^2}^2 
    \le 
    2 \left\| \hat{v}^{(t)} - \tilde{v}^{(t)} \right\|_{L^2}^2 
    + 2 \left\| \tilde{v}^{(t)} - v^{(t)} \right\|_{L^2}^2.
\end{equation}
For the second term, $\tilde{v}^{(t)}$ points to the uncompressed barycenter $\nu_{\mathrm{avg}}^{(t)}$, while $v^{(t)}$ points to the OT-compressed barycenter $\mu_{\mathrm{avg}}^{(t)}$. By the bounded compression property established in \eqref{eq:ot_compression_bound}, this term is bounded by $\delta_n^2(S)$. 

For the first term, the virtual map $\hat{v}^{(t)}$ aggregates gradients evaluated at the synchronized global state $\mu_{\mathrm{avg}}^{(t-1)}$, while the ideal map $\tilde{v}^{(t)}$ aggregates gradients evaluated at the scattered local states $\mu_u^{(t-1)}$. By Jensen's inequality and the $L$-Lipschitz property of the gradients (Assumption~\ref{assum:w2_smooth_rev}), we have:
\begin{align}
    \left\| \hat{v}^{(t)} - \tilde{v}^{(t)} \right\|_{L^2}^2 
    &\le 
    \frac{\eta^2 L^2}{m} \sum_{u=1}^m W_2^2(\mu_{\mathrm{avg}}^{(t-1)}, \mu_u^{(t-1)}).
\end{align}
Summing these bounds and taking the expectation completes the proof.
\end{proof}

\begin{theorem}[Convergence to a Wasserstein stationarity neighborhood]
\label{thm:w2_stationarity_neighborhood}
Suppose Assumptions~\ref{assum:w2_smooth_rev}--\ref{assum:graph_mixing} and the bounded compression property \eqref{eq:ot_compression_bound} hold. 
Let the learning rate satisfy $\eta\le 1/L$. 
Then after $T$ communication rounds, the decentralized OT-based prompt aggregation procedure satisfies
{\small
\begin{align*}
    &\frac{1}{T}
    \sum_{t=1}^{T}
    \mathbb{E}
    \left[
    \left\|
    \nabla_{W_2}\mathcal{F}
    \left(
    \mu_{\mathrm{avg}}^{(t-1)}
    \right)
    \right\|_{L^2(\mu_{\mathrm{avg}}^{(t-1)})}^2
    \right]
    \le
    \frac{
    2\left(
    \mathcal{F}(\mu_{\mathrm{avg}}^{(0)})-\mathcal{F}^*
    \right)
    }{\eta T}
    \\ &+
    \mathcal{O}
    \left(
    \frac{\eta^2 L^2 G^2}{1-\mathcal{D}\rho^2}
    \right)
    +
    \mathcal{O}
    \left(
    \frac{L^2 \delta_n^2(S)}{1-\mathcal{D}\rho^2}
    +
    \frac{\delta_n^2(S)}{\eta^2}
    \right).
\end{align*}
}
\end{theorem}

\begin{proof}
Because $\mu_{\mathrm{avg}}^{(t)}$ is a Wasserstein barycenter, we expand the $L$-smooth global functional $\mathcal{F}$ around the previous state $\mu_{\mathrm{avg}}^{(t-1)}$ using \eqref{eq:w2_smoothness_bound}:
\begin{align*}
    \mathbb{E}[\mathcal{F}(\mu_{\mathrm{avg}}^{(t)})] 
    \le 
    \mathbb{E}[\mathcal{F}(\mu_{\mathrm{avg}}^{(t-1)})] 
    &+ \mathbb{E} \left[ \langle \nabla_{W_2} \mathcal{F}(\mu_{\mathrm{avg}}^{(t-1)}), v^{(t)} \rangle_{L^2} \right] 
    \\&+ \frac{L}{2} \mathbb{E} \left[ \|v^{(t)}\|_{L^2}^2 \right].
\end{align*}
Applying the polarization identity to the inner product with the virtual gradient $\hat{v}^{(t)} = -\eta \nabla_{W_2} \mathcal{F}(\mu_{\mathrm{avg}}^{(t-1)})$ and utilizing the condition $\eta \le 1/L$ to discard the non-positive $\|v^{(t)}\|^2$ coefficient, we obtain the descent inequality:
\begin{align*}
    \mathbb{E}[\mathcal{F}(\mu_{\mathrm{avg}}^{(t)})] 
    &\le 
    \mathbb{E}[\mathcal{F}(\mu_{\mathrm{avg}}^{(t-1)})] 
    - \frac{\eta}{2} \mathbb{E} \left\| \nabla_{W_2} \mathcal{F}(\mu_{\mathrm{avg}}^{(t-1)}) \right\|_{L^2}^2 
    \\&+ \frac{1}{2\eta} \mathbb{E} \left\| \hat{v}^{(t)} - v^{(t)} \right\|_{L^2}^2.
\end{align*}
Substituting the tracking error bound from Lemma~\ref{lem:tracking_error} and recognizing that the trailing summation $\frac{1}{m}\sum_u W_2^2$ is exactly our asymptotic network consensus error $\varepsilon^{(t-1)}$ defined in \eqref{eq:appendix:consensus_error_def} and bounded in Theorem~\ref{thm:global_consensus_bound}, we find:
\begin{align*}
    \mathbb{E}[\mathcal{F}(\mu_{\mathrm{avg}}^{(t)})] 
    &\le 
    \mathbb{E}[\mathcal{F}(\mu_{\mathrm{avg}}^{(t-1)})] 
    - \frac{\eta}{2} \mathbb{E} \left\| \nabla_{W_2} \mathcal{F}(\mu_{\mathrm{avg}}^{(t-1)}) \right\|_{L^2}^2 
    \\ &+ \mathcal{O}\left( \frac{\eta^3 L^2 G^2}{1-\mathcal{D}\rho^2} + \frac{\eta L^2 \delta_n^2(S)}{1-\mathcal{D}\rho^2} \right) 
    + \frac{\delta_n^2(S)}{\eta}.
\end{align*}
Telescoping across $T$ rounds and dividing by $\eta T/2$ yields the final result.
\end{proof}

\paragraph{Interpretation.}
Theorem~\ref{thm:w2_stationarity_neighborhood} demonstrates that the procedure converges to a neighborhood of stationarity. The neighborhood size is determined by the gradient variance $G^2$ and the aggregation noise $\delta_n^2(S)$ induced by the fixed-budget OT constraint. Notably, the $\delta_n^2(S)/\eta^2$ penalty indicates that the approximation error in the communication step sets a floor on the achievable stationarity, a common characteristic in decentralized optimization with lossy compression.

\section{Additional Experiment Settings}\label{appendix:additional_settings}
\subsection{Datasets and Partitions}\label{appendix:data_model}

Our experiments are conducted on two synthetic, multi-domain datasets constructed by pooling together several heterogeneous image classification benchmarks. These composite datasets are designed to simulate realistic federated learning scenarios in which clients not only disagree on class distributions, but also hold data drawn from different visual domains.

\paragraph{\textbf{4-dataset~\citep{weng2024probabilistic}:} } The first composite dataset combines four sub-datasets: MNIST-M \cite{ganin2016domain}, Fashion-MNIST, CINIC-10 \cite{darlow2018cinic}, and MMAFEDB (available on Kaggle)\footnote{\url{https://www.kaggle.com/datasets/yuulind/mmafedb-clean}}. These sub-datasets span diverse visual domains ranging from colorized digit images to fashion items, natural scene photographs, and facial expressions. Together they constitute 37 classes (10 + 10 + 10 + 7, respectively). For the training partition, we sample 30,000 examples per sub-dataset, yielding 120,000 training images in total. For the test partition, we sample 2,500 examples per sub-dataset, yielding 10,000 test images in total. We simulate $m = 40$ clients by assigning 10 clients to each sub-dataset, so that each client only ever holds data from one visual domain.
    
\paragraph{\textbf{5-dataset~\citep{wang2022learning}:}} The second composite dataset combines five sub-datasets: CIFAR-10~\citep{krizhevsky2009learning}, MNIST~\citep{lecun1998mnist}, Fashion-MNIST, SVHN~\citep{netzer2011reading}, and NotMNIST~\citep{bulatov2011notmnist}, each contributing 10 classes for a total of 50 classes. This collection spans natural image classification, handwritten digit recognition, grayscale fashion item recognition, street-view digit recognition, and printed character recognition, covering a broad range of low-level statistics and label semantics. For the training partition, we sample 20,000 examples per sub-dataset, yielding 100,000 training images in total. For the test partition, we sample 2,000 examples per sub-dataset, yielding 10,000 test images in total. We simulate $m = 50$ clients by assigning 10 clients to each sub-dataset.

\paragraph{Heterogeneous Partition.} We partition each sub-dataset independently among its 10 assigned clients using a $\text{Dirichlet}(\alpha \cdot \mathbf{1}_s)$ distribution over the $s$-class simplex, where $s$ is the number of classes in that sub-dataset. Each client receives a proportion vector drawn from this distribution, controlling what fraction of each class is allocated to that client. Smaller values of $\alpha$ produce more skewed, heterogeneous distributions. We run experiments with $\alpha = 0.1$, which produces high heterogeneity, and $\alpha = 0.5$, which produces moderate heterogeneity. Because the Dirichlet draws are applied independently per sub-dataset using a shared random seed, the resulting distributions are statistically comparable across sub-datasets within the same run.

\paragraph{Extreme Non-iid Partition.} We additionally evaluate a manual extreme-heterogeneity setting that produces maximally imbalanced local datasets. For each sub-dataset, 99\% of the data belonging to each class is assigned exclusively to one designated client, while the remaining 1\% is distributed among non-designated clients via a symmetric Dirichlet distribution with concentration parameter $\alpha = 1$. Since each sub-dataset contains exactly 10 classes and is partitioned among exactly 10 clients, this scheme results in a bijective assignment in which every client is dominated by exactly one class, with only trace amounts of the remaining classes present in its local dataset.

The above partitioning schemes are applied only to the training split. Evaluation is performed globally on the full held-out test partition of each composite dataset.
\begin{figure*}[ht!]
    \centering
    \includegraphics[width=\linewidth]{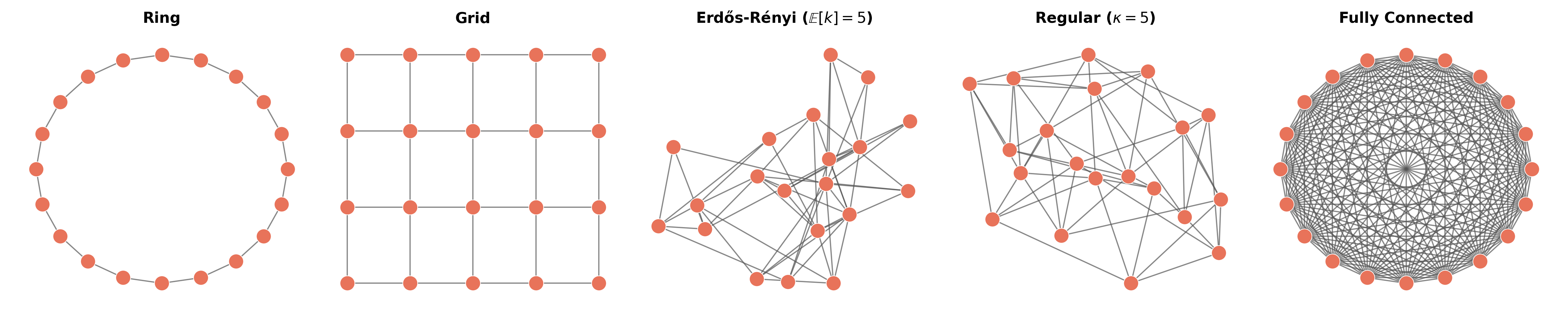}
    \caption{Communication topologies used in our experiments ($m = 20$ nodes shown for clarity).
    From left to right: Ring, Grid, Erdős-Rényi, Regular ($\kappa=5$), and Fully Connected.}
    \label{fig:topologies}
\end{figure*}
\subsection{Hyperparameter Settings}\label{appendix:hyperparams}
For details regarding Prompt-tuning protocol and DFL baselines, please refer to Appendix \ref{appx:add_prem}.
\paragraph{Shared settings.}
All methods share the same backbone, optimizer, batch size, number of
communication rounds, prompt configuration, communication topology,
participation rate, and evaluation schedule; only the method-specific
parameters listed below differ. Concretely, every method uses a frozen
ViT-B/32 backbone with $n = 10$ learnable prompt tokens of dimension $d = 768$
prepended to the patch-embedding sequence, optimized with Adam at learning rate
$\eta = 10^{-4}$ and batch size $16$. Training runs for $40$ communication
rounds. All clients participate in every round
(full participation). The number of clients is $40$ for FourDataset and $50$
for FiveDataset, with the default communication topology being a time-varying
$\kappa$-regular graph ($\kappa = 4$ for FourDataset, $\kappa = 5$ for
FiveDataset).

\paragraph{D-FROST.}
D-FROST runs $5$ local epochs per round with Adam. The OT-based \texttt{Merge}
step runs $S = 50$ alternating inner iterations, with entropy regularization
weight $\varepsilon = 0.01$, spatial scale $\sigma^2 = 1.0$, and $L_2$
regularization weight $\lambda = 0.001$. $\varepsilon = 0.01$ follows standard
entropic-OT practice of choosing small $\varepsilon$ for sharp, stable
transport~\citep{cuturi2013sinkhorn,peyre2019computational}. We also explore
$\varepsilon$ annealing~\citep{schmitzer2019stabilized} in
Appendix~\ref{appendix:entropy}. Due to the low computation cost of the OT-based
\texttt{Merge} step (see Appendix~\ref{appendix:comp_cost}), we can afford a
large number of inner iterations $S$ to drive the merge error $\delta_n^2(S)$
low and ensure the solver converges. We found that $S = 50$ is
sufficient.

\paragraph{DFedAvgM-PT.}
DFedAvgM-PT runs $5$ local epochs per round using SGD with momentum
$\beta = 0.99$.

\paragraph{DFedSAM-PT.}
DFedSAM-PT runs $5$ local epochs per round using adaptive SAM with perturbation
radius $\rho = 0.01$, following prior work.

\paragraph{D-PSGD-PT.}
D-PSGD-PT runs $1$ local epoch per round using SGD with momentum $0$.

\paragraph{Implementation Details.}
Experiments, including the runtime measurements in Appendix \ref{appendix:comp_cost},
were conducted on a Linux workstation running Ubuntu 20.04 LTS, equipped with an
Intel Xeon E5-2697 v4 CPU @ 2.30\,GHz (18 cores, 36 threads), 384\,GB RAM, and a NVIDIA RTX A6000 GPU (48\,GB VRAM). Our implementation is based on PyTorch~2.0 with
CUDA~12.2.

\subsection{Network Topologies}\label{appendix:topologies}
We evaluate all methods on five undirected communication topologies of varying connectivity.
Each topology is instantiated over $m$ clients and represented by a symmetric doubly stochastic mixing matrix $W \in \mathbb{R}^{m \times m}$, where $W_{uv} > 0$ only if $u = v$ or $(u, v)$ is an edge.
The spectral mixing factor $\rho = \|W - \frac{1}{m}\mathbf{1}\mathbf{1}^\top\|_2$ characterizes how quickly information spreads: a smaller $\rho$ indicates faster mixing and tighter Wasserstein consensus (Theorem~2).
The five topologies, ordered from sparsest to densest ($\rho$ decreasing), are as follows.

\paragraph{Ring.}
Each client connects to exactly two neighbors arranged in a cycle.
With $m$ clients, every node has degree $2$, making the ring the sparsest topology and the one with the largest mixing factor $\rho$. Node identities are
randomly permuted at each communication round while the cyclic structure is
preserved.

\paragraph{Grid.}
Clients are arranged in a two-dimensional lattice with $r \times c = m$ cells, where $r = \max\{k \leq \lfloor\sqrt{m}\rfloor : k \mid m\}$ and $c = m / r$.
For FiveDataset ($m = 50$) this yields a $5 \times 10$ lattice; for FourDataset ($m = 40$) a $5 \times 8$ lattice.
Interior nodes have degree $4$, boundary nodes degree $3$ or $2$, and corner nodes degree $2$.
Each round, node identities are randomly permuted at each communication round while the lattice structure is preserved, so neighbor assignments change over time.

\paragraph{Erdős-Rényi.}
Each pair of clients is connected independently with probability
$p = \kappa / (m - 1)$, matching the expected degree of the $\kappa$-regular
topology ($\kappa = 5$, $p \approx 0.1$ for $m = 50$).
Unlike the regular graph, the ER graph has degree variance, so some nodes
acquire fewer links than others.
A fresh ER graph is resampled each round.

\paragraph{Regular (default).}
Each client is connected to exactly $\kappa$ randomly chosen neighbors, forming a $\kappa$-regular graph.
We use $\kappa = 5$ for FiveDataset and $\kappa = 4$ for FourDataset.
As our \emph{standard} topology, we employ a \emph{time-varying} $\kappa$-regular graph: at each communication round $t$, a fresh random $\kappa$-regular graph $G^{(t)} = (V, E^{(t)})$ is independently sampled, so the neighbor set of each client changes every round.
This models realistic wireless or peer-to-peer networks with transient link availability.

\paragraph{Fully Connected.}
Every pair of clients communicates directly, yielding a complete graph of degree $m - 1$.
The mixing matrix is $W = \frac{1}{m}\mathbf{1}\mathbf{1}^\top$, giving $\rho = 0$ and perfect one-hop consensus.
This topology represents an idealized upper bound on connectivity.

\paragraph{Mixing matrix construction.}
For all topologies, the mixing matrix $W$ is symmetric and doubly stochastic for any undirected graph, satisfying Assumption~4.



\section{Additional Experiment Results}\label{appendix:additional_results}
Tables~\ref{tab:accuracy_fourdataset} and~\ref{tab:accuracy_fivedataset} report
final test accuracy on FourDataset and FiveDataset under the two Dirichlet splits
($\alpha = 0.5$, $\alpha = 0.1$) and the extreme non-IID partition, and
Table~\ref{tab:topology} breaks down FiveDataset ($\alpha = 0.1$) across the five
communication topologies. D-FROST achieves the best accuracy in every setting,
and its margin over the strongest baseline widens as heterogeneity increases. The gains are
also consistent across all topologies, confirming that the advantage of OT-based
merging does not depend on a particular graph structure.

Figure~\ref{fig:noniid_fourdataset} shows the performance comparison on FourDataset. Specifically, it reaches $69.06\%$, $63.94\%$, and $58.24\%$, improving over the strongest baseline (DFedAvgM) by $5.78$, $4.22$, and $12.19$ points, respectively. 
\begin{table}[htbp]
\centering
\caption{Test Accuracy (\%) achieved on the \textbf{Fourdataset} by D-FROST and other baselines.}
\label{tab:accuracy_fourdataset}
\begin{tabular}{|l|c|c|c|}
\hline
 Algorithm & $\alpha = 0.5$ & $\alpha = 0.1$ & Ex. non-iid \\
\hline
\textsc{D-PSGD-PT}   & 24.47 & 21.45 & 12.38 \\
\textsc{DFedSAM-PT}  & 62.30 & 59.49 & 43.29 \\
\textsc{DFedAvgM-PT} & 63.28 & 59.72 & 46.05 \\
\textsc{D-FROST}  & \textbf{69.06} & \textbf{63.94} & \textbf{58.24} \\
\hline
\end{tabular}
\end{table}

\begin{table}[htbp]
\centering
\caption{Test Accuracy (\%) achieved on the \textbf{Fivedataset} by D-FROST and other baselines.}
\label{tab:accuracy_fivedataset}
\begin{tabular}{|l|c|c|c|}
\hline
 Algorithm & $\alpha = 0.5$ & $\alpha = 0.1$ & Ex. non-iid \\
\hline
\textsc{D-PSGD-PT}    & 23.23 & 18.14 & 9.92 \\
\textsc{DFedSAM-PT}  & 70.31 & 66.74 & 45.29 \\
\textsc{DFedAvgM-PT} & 75.28 & 70.88 & 51.89 \\
\textsc{D-FROST}  & \textbf{81.95} & \textbf{79.36} & \textbf{70.46} \\
\hline
\end{tabular}
\end{table}

\begin{table}[htbp]
\centering
\caption{Test accuracy (\%) in various network topologies on FiveDataset under Dirichlet $\alpha = 0.1$. All methods follow the prompt-tuning (PT) protocol.}
\label{tab:topology}
\begin{tabular}{lccccc}
\toprule
Algorithm & Ring & Grid & Erdos & Regular & Full \\
\midrule
D-PSGD    & 16.42 & 17.69 & 18.29 & 18.14 & 19.00 \\
DFedSAM   & 64.21 & 65.85 & 65.97 & 66.74 & 69.80 \\
DFedAvgM  & 67.37 & 68.57 & 70.64 & 70.88 & 72.74 \\
D-FROST   & \textbf{75.88} & \textbf{77.84} & \textbf{78.96} & \textbf{79.36} & \textbf{82.28} \\
\bottomrule
\end{tabular}
\end{table}
\begin{figure*}[t]
  \centering
  \includegraphics[width=\textwidth]{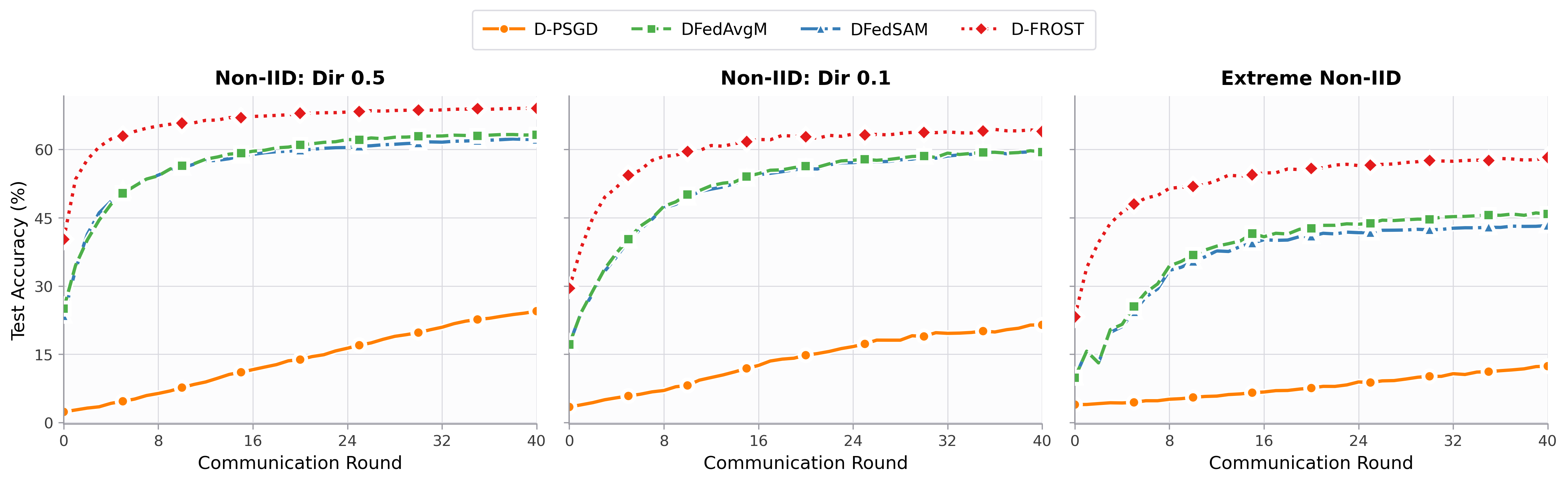}
  \caption{Test accuracy of all methods on \textbf{FourDataset} (40 clients)
           under three non-IID settings: Dirichlet $\alpha{=}0.5$,
           Dirichlet $\alpha{=}0.1$, and the extreme non-IID partition.}
  \label{fig:noniid_fourdataset}
\end{figure*}

\section{Factors Impacting Network Consensus Error}
\label{appendix:consensus_factors}
\begin{figure}[h]
    \centering
    \includegraphics[width=0.4\textwidth]{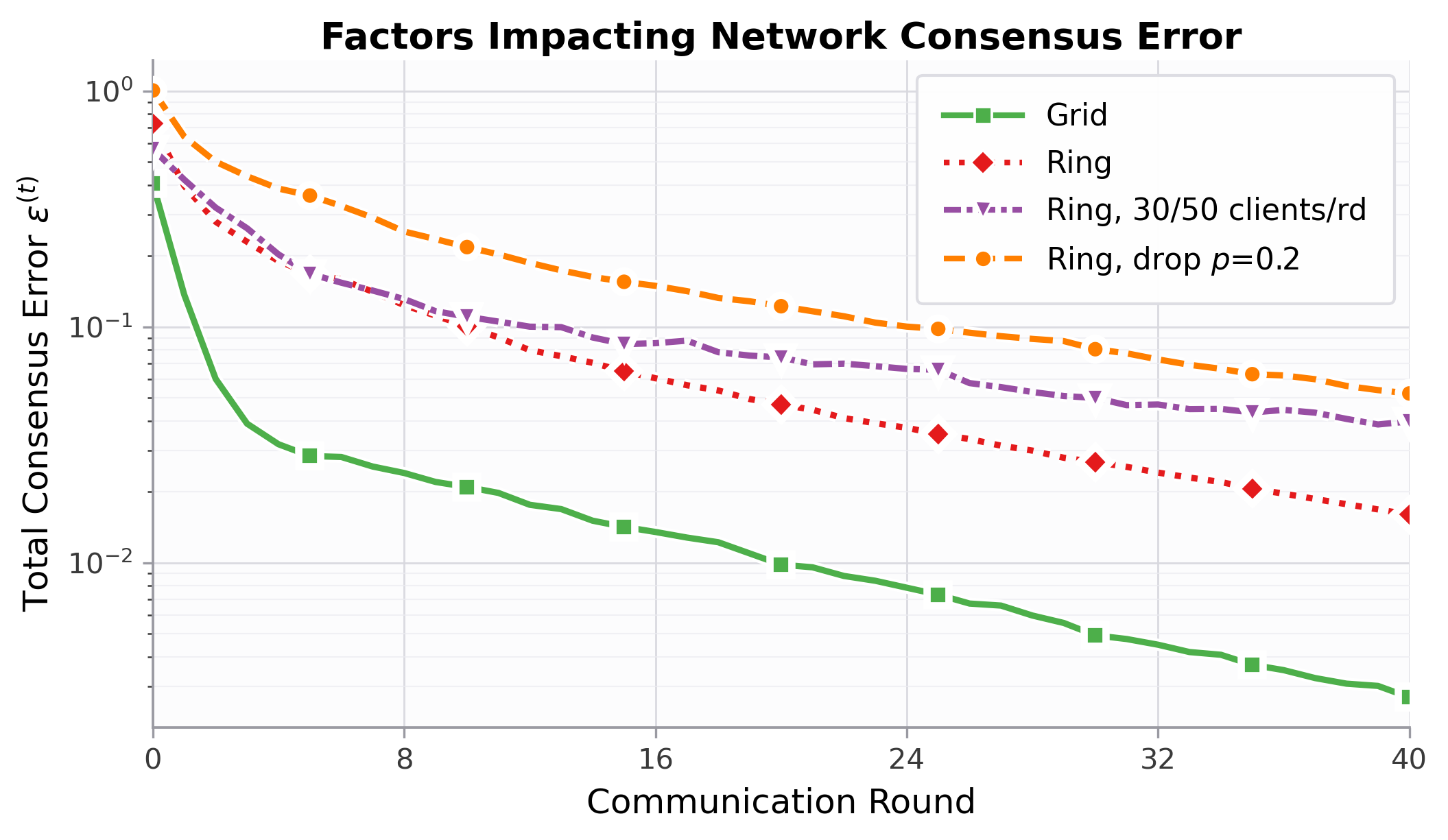}
    \caption{Total Wasserstein consensus error of D-FROST across all clients on FiveDataset under Dirichlet $\alpha = 0.1$.}
    \label{fig:consensus_factors}
\end{figure}
Theorem~\ref{thm:global_consensus_bound} bounds the network consensus error by
$\varepsilon^{(t)} \le \beta/(1-\alpha) = O\!\big(\tfrac{\delta_n^2(S)+\eta^2 G^2}{1-D\rho^2}\big)$,
whose denominator is governed by the mixing factor $\rho$: weaker mixing (larger $\rho$)
enlarges the consensus neighborhood. We probe this prediction by degrading the network along
three axes: topology connectivity, link
dropout, and partial participation, and tracking the total consensus error over the communication rounds (Figure~\ref{fig:consensus_factors}). The curves separate exactly
as the bound predicts. The denser Grid mixes fastest and sits well below the sparse Ring at
every round. Introducing link dropout ($p=0.2$) on Ring topology
degrades its effective mixing and lifts its curve to the top of the plot, while partial
participation ($30$ of $50$ clients per round) slows mixing and places it between the clean Ring
and the dropout case. The result confirms that every factor weakening graph mixing enlarges the consensus floor
through the same $1/(1-D\rho^2)$ mechanism. Across all four settings, however, $\varepsilon^{(t)}$
still contracts monotonically, showing that the OT-based merge keeps the network converging even
under sparse, unreliable, or partially participating topologies.

\section{Cost Analysis of D-FROST}
\label{appendix:comp_cost}

We analyze the per-round cost the OT-based \texttt{Merge} step
(Algorithm~\ref{alg:dec_ot}),
asymptotically and in wall-clock time.

\paragraph{Setup.}
At round $t$, client $u$ forms the neighborhood collection
$\Omega_u^{(t)} = \tilde\omega_u^{(t)} \uplus \biguplus_{v\in\mathcal{N}(u)} \tilde\omega_v^{(t)}$
of size $N_u^{(t)} := |\Omega_u^{(t)}|$ (Eq.~\eqref{eq:Omega_u}); with $n$ prompts
per client and a $\kappa$-regular topology, $N_u^{(t)} = (\kappa+1)\,n$. The
representative set $\Phi=\{\phi_i\}_{i=1}^{n}$ stays size $n$ throughout, so the merged
prompt-set size is preserved.

\paragraph{Complexity.}
Each of the $S$ inner iterations runs three closed-form steps: the cost matrix
$C\in\mathbb{R}^{N_u^{(t)}\times n}$ (Eq.~\eqref{eq:ot_cost}), the transport plan $P$
(Eq.~\eqref{eq:P_subproblem}), and the barycenter update (Eq.~\eqref{eq:update_phi}). It is dominated
by the two matrix products $\Omega_u\Phi^\top$ and $P^\top\Omega_u$. This gives
\begin{equation}
    \mathcal{T}_{\mathrm{OT}} = O\!\big(S\,(\kappa+1)\,n^2 d\big)
    \label{eq:ot_complexity}
\end{equation}
quadratic in the prompt budget $n$ and linear in $\kappa+1$, $d$, and $S$. With our
settings ($S{=}50$, $\kappa{=}5$, $n{=}10$, $d{=}768$, so $N_u^{(t)}{=}60$), the merge
costs $\approx 4.6\times 10^{7}$ MACs, which is negligible compared to a forward/backward pass of the
frozen ViT-B/32 model. 

\paragraph{Wall-clock overhead.}
Table~\ref{tab:wall_clock} reports per-client, per-round times on FiveDataset
($\kappa{=}5$, $n{=}10$). The OT merge takes $0.63$\,s versus $0.26$\,s for index-wise
averaging. The OT merge step only accounts for $\approx 1.5\%$ of round time, which
is dominated by local training ($\approx 42.3$\,s). 
\begin{table}[h]
\centering
\caption{Per-round wall-clock times (seconds) for one client on FiveDataset, using a
time-varying $\kappa$-regular graph ($\kappa{=}5$). \emph{Agg.}\ is the \texttt{Merge}
step only; \emph{Round} is total (train $+$ merge).}
\label{tab:wall_clock}
\begin{tabular}{lccc}
\toprule
Method & Train (s) & Agg.\ (s) & Round (s) \\
\midrule
D-PSGD-PT      & $8.42$  & $0.26$ & $8.68$ \\
DFedAvgM-PT    & $42.43$ & $0.26$ & $42.69$ \\
D-FROST (ours) & $42.30$ & $0.63$ & $42.94$ \\
DFedSAM-PT     & $80.51$ & $0.25$ & $80.75$ \\
\bottomrule
\end{tabular}
\end{table}

\paragraph{Communication overhead.}
The OT cost matrix $C$ and transport plan $P$ are computed locally and \emph{never
transmitted}; clients exchange only their updated prompt sets $\tilde\omega_v^{(t)}$
($n$ prompts of dimension $d$), exactly as the parametric baselines do. The OT-based
merge therefore incurs \textbf{no extra communication cost} over index-wise averaging.

\paragraph{Accuracy-efficiency trade-off.} Figure~\ref{fig:speed_vs_accuracy} plots final accuracy against average wall-clock time per round per client on FiveDataset. D-FROST reaches the highest accuracy ($79.36\%$) at $42.93$\,s/round and lies on the Pareto frontier. DFedSAM-PT is the most expensive method at $80.75$\,s/round yet reaches only $66.74\%$. The cheaper baselines (D-PSGD-PT at $8.68$\,s/round and DFedAvgM-PT at $42.69$\,s/round) run faster per round but plateau far below D-FROST in accuracy. Furthermore, while the OT-based aggregation procedure introduces a marginal computational overhead compared to index-wise averaging, it yields substantial accuracy gains without incurring any extra communication cost. We analyze D-FROST's overhead in detail in Appendix \ref{appendix:comp_cost}.

\begin{figure}[ht!]
    \centering
    \includegraphics[width=0.45\textwidth]{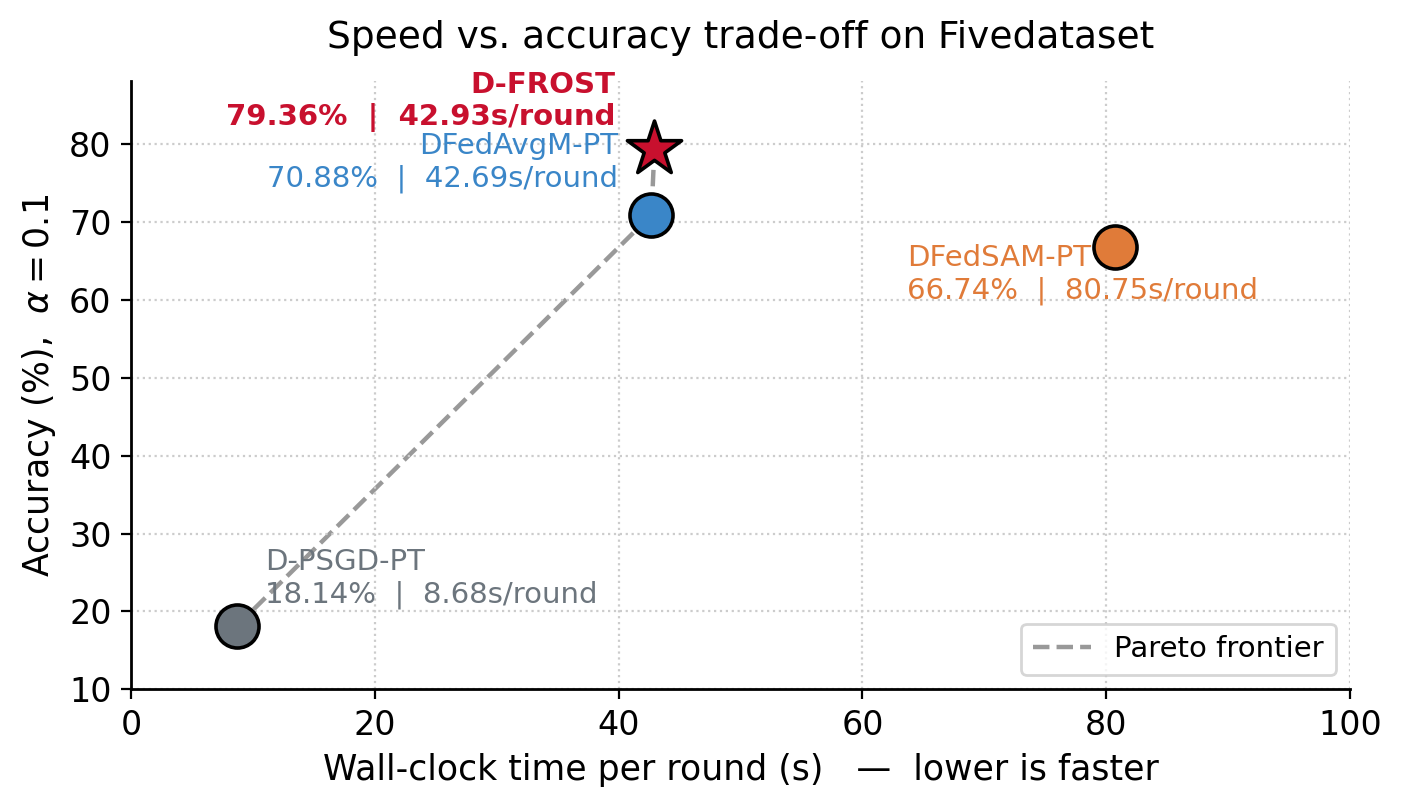}
    \caption{Accuracy versus per-round wall-clock per client cost on FiveDataset ($\alpha=0.1$).}
    \label{fig:speed_vs_accuracy}
\end{figure}

\section{Ablation Studies}\label{appx:ablation}
\subsection{Impact of Client Prompt Budget $n$}\label{appendix:prompt_budget}
\begin{figure}[h!]
    \centering
    \includegraphics[width=0.4\textwidth]{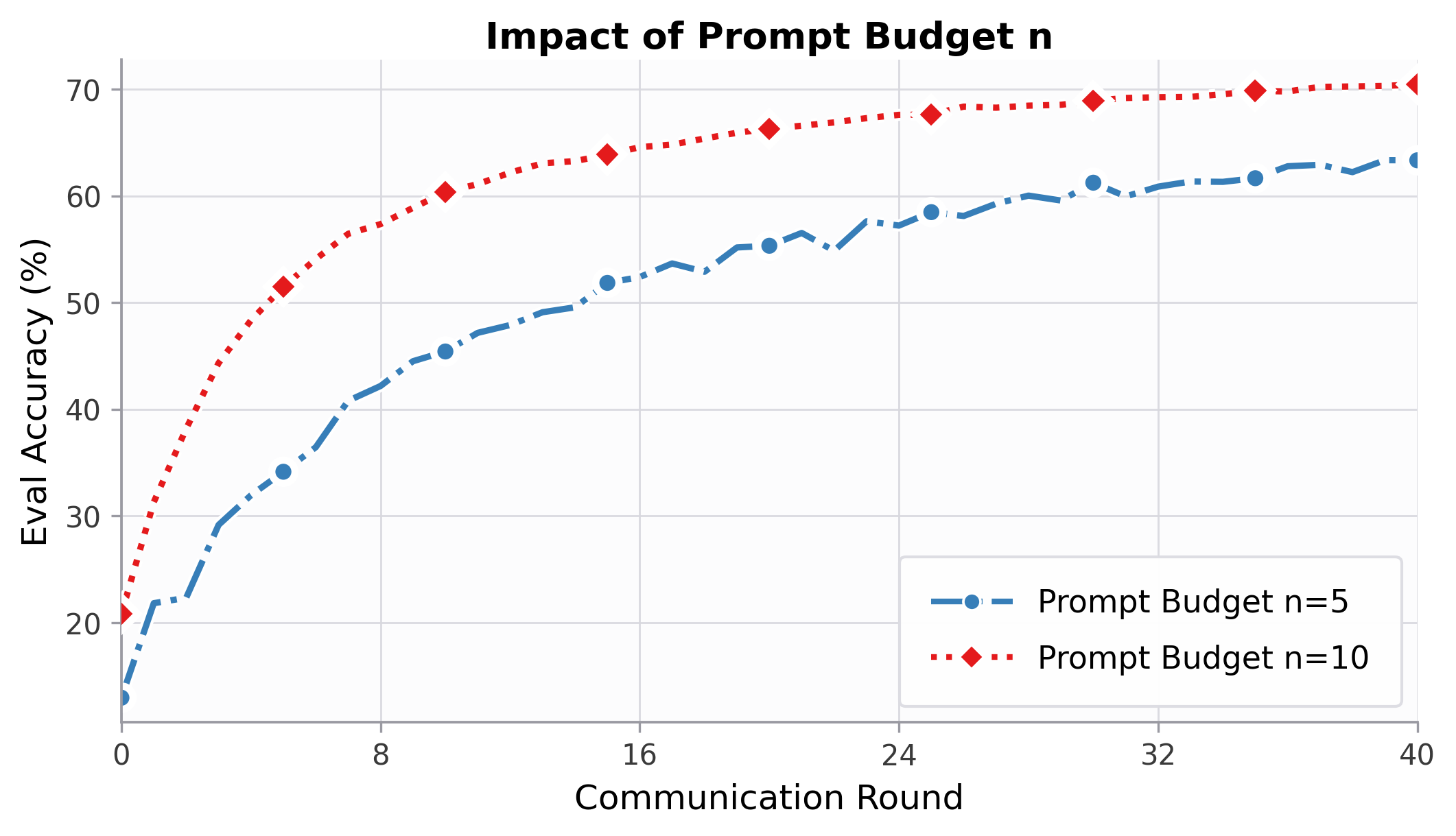}
    \caption{Impact of client prompt budget $n$.}
\end{figure}

The prompt budget $n$ is the number of representative prompts each client retains after the
OT-based \texttt{Merge} step, i.e.\ the size of $\Phi=\{\phi_i\}_{i=1}^{n}$. It controls how
faithfully the merged set summarizes the neighborhood collection $\Omega_u^{(t)}$: a larger
$n$ lowers the merge error $\delta_n^2(S)$ (Eq.~\eqref{eq:ot_compression_bound}), which by
Theorems~\ref{thm:global_consensus_bound} and~\ref{thm:w2_stationarity_neighborhood} tightens
both the consensus and stationarity neighborhoods and thus raises attainable accuracy.

On FiveDataset under the extreme non-IID partition, shrinking the budget to $n=5$ drops accuracy from
$70.46\%$ to $63.35\%$. We use $n=10$ as the default in all main experiments.
\subsection{Impact of Different Frozen Backbones}\label{appendix:backbones}
\begin{figure}[h]
    \centering
    \includegraphics[width=0.4\textwidth]{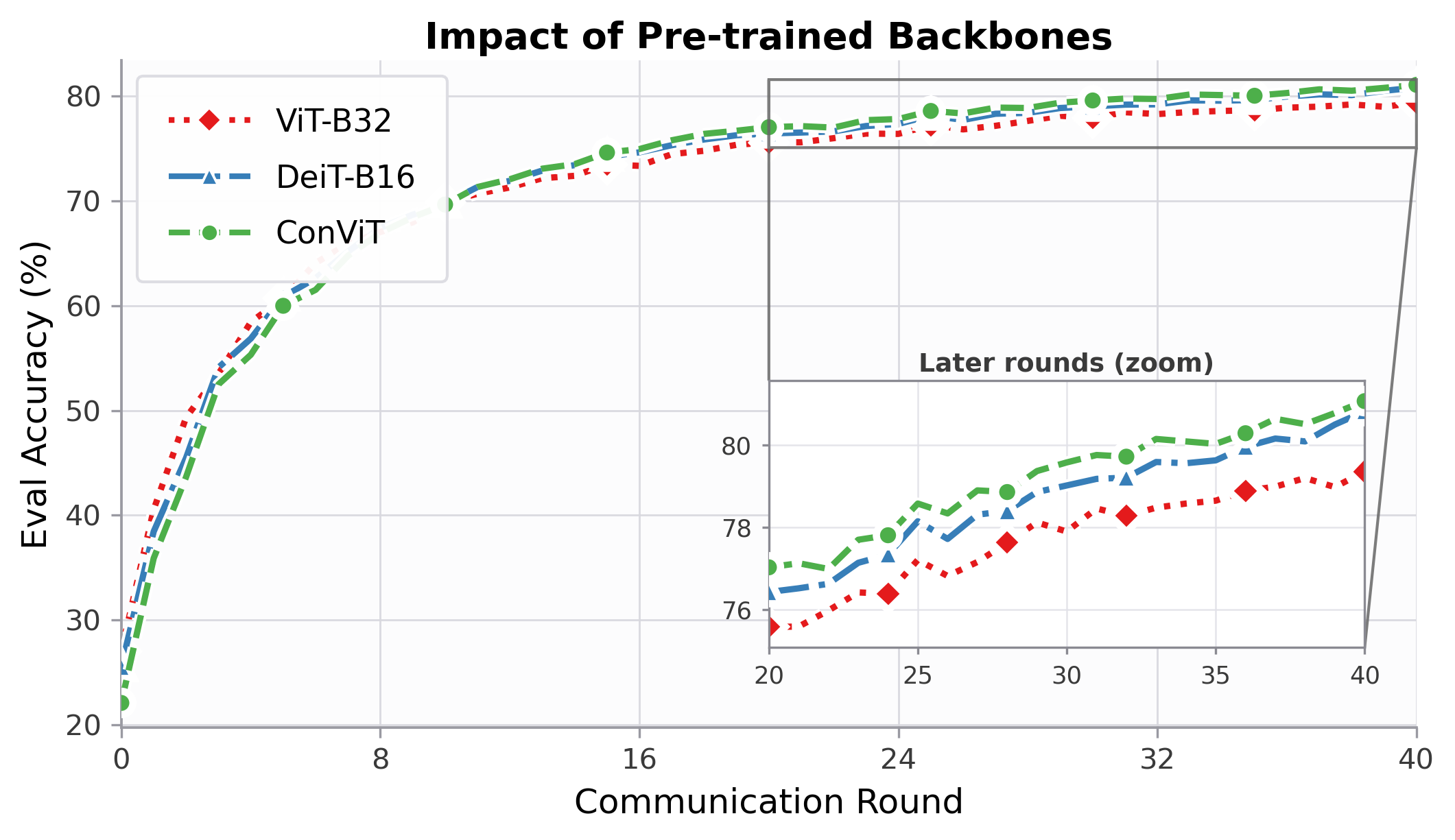}
    \caption{Impact of different pre-trained backbones.}
    \label{fig:backbones}
\end{figure}
D-FROST treats the backbone as a frozen feature extractor and performs all aggregation in
the $d$-dimensional prompt embedding space (Eq.~\eqref{eq:ot_cost}). Since the OT-based
\texttt{Merge} never inspects the backbone weights or architecture, the method transfers across ViT architectures. Figure~\ref{fig:backbones} verifies this on FiveDataset
($\alpha=0.1$) with three frozen backbones: ViT-B/32,
DeiT-B/16~\citep{touvron2021training} and ConViT-Base~\citep{d2021convit}.
\subsection{Comparison with Improved Baselines}
\label{appendix:improved_baselines}
\begin{figure}[h]
    \centering
    \includegraphics[width=0.4\textwidth]{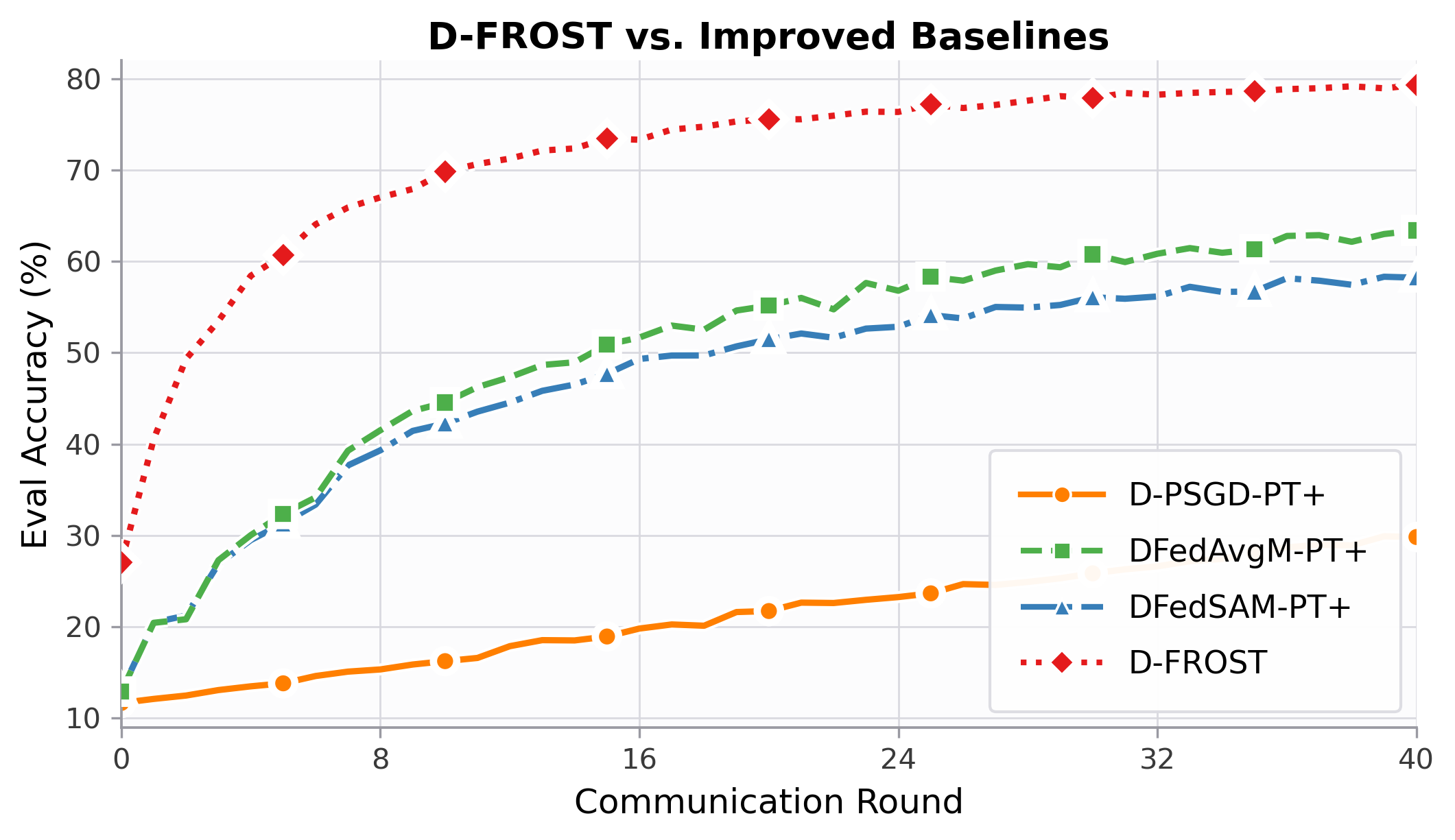}
    \caption{D-FROST versus the improved baselines (denoted $+$) on FiveDataset under the
    extreme non-IID partition. Each baseline is augmented with
    an L2P-style prompt-selection mechanism~\citep{wang2022learning}.}
    \label{fig:improved_baselines}
\end{figure}

Beyond the results in the main text, we compare D-FROST against a stronger set of baselines.
Each DFL baseline is augmented with a client-specific prompt-selection
mechanism~\citep{wang2022learning, weng2024probabilistic}, which lets every client contextualize its local data by
selecting the most relevant prompts from a shared pool rather than collapsing distinct contexts
into the same prompts. Concretely, each client maintains a prompt pool of size $20$ and, for each input, selects $10$ most relevant
prompts from this pool to prepend using a query mechanism. Setting the pool size to $10$ recovers the original baseline, since every input then uses
all $10$ prompts and no selection takes place. We denote these improved variants with a $+$:
D-PSGD-PT$+$, DFedAvgM-PT$+$, and DFedSAM-PT$+$.

Figure~\ref{fig:improved_baselines} reports the comparison on FiveDataset under the extreme
non-IID partition. Prompt selection substantially raises the baselines over their original
counterparts, yet D-FROST still outperforms all of them by a clear margin at every round: it
separates within the first few rounds and converges to roughly $79\%$, while the strongest
improved baseline, DFedAvgM$+$, plateaus near $63\%$, followed by DFedSAM$+$ ($\approx 58\%$)
and D-PSGD$+$ ($\approx 30\%$).

\subsection{Entropy Value Selection Strategies }\label{appendix:entropy}
\begin{figure}[h]
    \centering
    \includegraphics[width=0.4\textwidth]{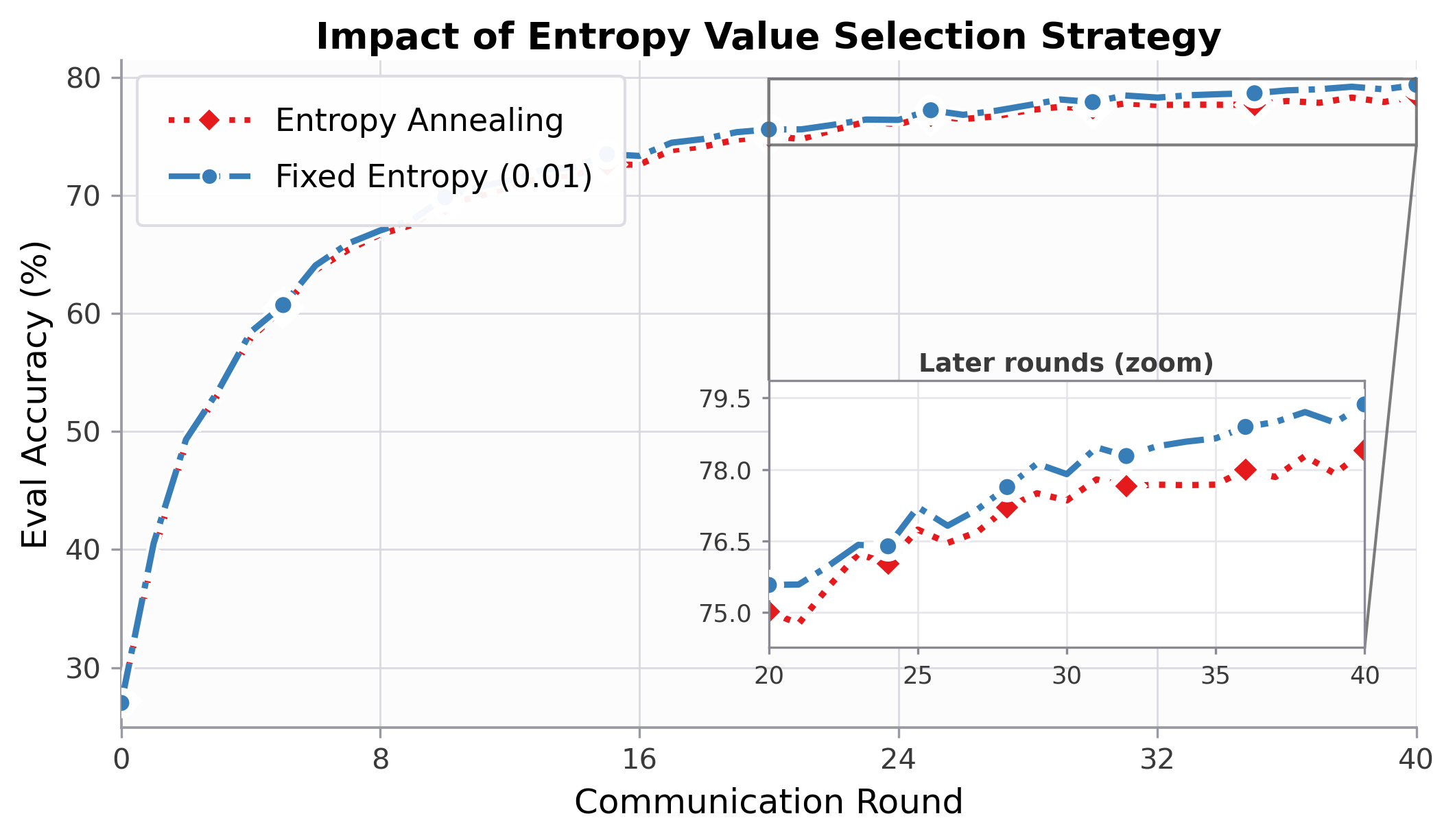}
    \caption{Impact of Entropy Value Selection Strategies}
    \label{fig:l2p_entropy_comparison}
\end{figure}

The entropy weight $\varepsilon$ in the OT merge objective~\eqref{eq:full_objective} scales the
regularizer $\varepsilon\sum_{a,i}P_{ai}(\log P_{ai}-1)$, the standard entropic regularization
of optimal transport~\citep{cuturi2013sinkhorn,peyre2019computational}: a small $\varepsilon$
keeps the plan close to the exact OT solution and yields sharp assignments. We experiment with
$\varepsilon$-scaling~\citep{schmitzer2019stabilized}, a geometric schedule that starts from a
large $\varepsilon_{\text{init}}$ and anneals it down to the target $\varepsilon=0.01$ over the
$S$ inner iterations. On FiveDataset under Dirichlet
$\alpha=0.1$ (Figure~\ref{fig:l2p_entropy_comparison}), annealing does not improve over the fixed schedule. For simplicity we therefore use a fixed $\varepsilon=0.01$ in all
experiments.

\end{document}